\documentclass[USenglish]{article}

\usepackage[nonatbib,final]{neurips_2021}

\usepackage[utf8]{inputenc} %
\usepackage[T1]{fontenc}    %
\usepackage{babel}
\usepackage{csquotes}
\usepackage{url}            %
\usepackage{booktabs}       %
\usepackage{amsmath,amsfonts,amsthm,amssymb}
\usepackage{nicefrac}       %
\usepackage{microtype}      %
\usepackage{xcolor}         %
\usepackage{mathtools}
\usepackage{thmtools,thm-restate}
\usepackage{bbm}  %
\usepackage{enumitem}
\usepackage{thm-restate}

\usepackage{subcaption}

\definecolor{citecolor}{RGB}{0,180,0}
\definecolor{linkcolor}{RGB}{180,0,0}
\definecolor{urlcolor}{RGB}{0,0,180}
\usepackage[colorlinks,citecolor=citecolor,linkcolor=linkcolor,urlcolor=urlcolor]{hyperref}

\usepackage[capitalize,noabbrev]{cleveref}

\usepackage[style=authoryear,sortcites,sorting=ynt,maxcitenames=2,maxbibnames=10,maxalphanames=5,useprefix,uniquelist=minyear,dashed=false]{biblatex}
\renewbibmacro{in:}{} %
\DefineBibliographyStrings{english}{urlseen={visited in}}
\DeclareFieldFormat[unpublished,misc]{title}{\mkbibquote{#1}}  %
\let\citet\textcite
\let\citep\parencite

\newcommand{\httpurl}[1]{\href{http://#1}{\nolinkurl{#1}}}
\newcommand{\httpsurl}[1]{\href{https://#1}{\nolinkurl{#1}}}

\DeclareMathOperator{\rank}{rank}

\newcommand{\norm}[1]{\left\lVert #1 \right\rVert}
\newcommand{\Norm}[1]{\lVert #1 \rVert}

\newcommand{\N}{\mathbb{N}}
\newcommand{\Z}{\mathbb{Z}}

\newcommand{\RR}{\mathbb{R}}

\newcommand{\NN}{\mathbb{N}}

\newcommand{\OP}{\mathit{op}}

\newcommand{\BP}{\mathit{BP}}

\newcommand{\cK}{\mathcal{K}}

\newcommand{\cD}{\mathcal{D}}
\newcommand{\cE}{\mathcal{E}}
\newcommand{\cS}{\mathcal{S}}

\newcommand{\bE}{\mathbb{E}}
\newcommand{\R}{\mathbb{R}}
\newcommand{\F}{\mathcal{F}}
\renewcommand{\L}{\mathcal{L}}
\newcommand{\I}{\mathcal{I}}

\DeclareMathOperator*{\Var}{\mathrm{Var}}
\DeclareMathOperator*{\E}{\bE}      %

\DeclareMathOperator{\rad}{rad}

\DeclareMathOperator{\Tr}{Tr}

\DeclareMathOperator*{\argmin}{arg\,min}

\newtheorem{theorem}{Theorem}
\newtheorem{corollary}{Corollary}
\newtheorem{lemma}{Lemma}

\theoremstyle{definition}

\newtheorem{remark}{Remark}

\title{Uniform Convergence of Interpolators: Gaussian Width, Norm Bounds and Benign Overfitting}

\author{%
    Frederic Koehler\thanks{These authors contributed equally.}\\MIT\\\texttt{fkoehler@mit.edu}
    \And Lijia Zhou$^*$\\University of Chicago\\\texttt{zlj@uchicago.edu}
    \And Danica J.\ Sutherland\\UBC and Amii\\\texttt{dsuth@cs.ubc.ca}
    \And Nathan Srebro\\TTI-Chicago\\\texttt{nati@ttic.edu}
  }

\begin{document}

\maketitle
\setcounter{footnote}{0}
\begin{center}
\vspace{-12mm}
{{Collaboration on the Theoretical Foundations of Deep Learning} (\httpsurl{deepfoundations.ai})}
\vspace{3mm}
\end{center}

\begin{abstract}
We consider interpolation learning in high-dimensional linear regression with Gaussian data, and prove a generic uniform convergence guarantee on the generalization error of interpolators in an arbitrary hypothesis class in terms of the class’s Gaussian width.  Applying the generic bound to Euclidean norm balls recovers the consistency result of \citet{bartlett2020benign} %
for minimum-norm interpolators, and confirms a prediction of \citet{junk-feats} %
for near-minimal-norm interpolators in the special case of Gaussian data.  We demonstrate the generality of the bound by applying it to the simplex, obtaining a novel consistency result for minimum $\ell_1$-norm interpolators (basis pursuit). Our results show how norm-based generalization bounds can explain and be used to analyze benign overfitting, at least in some settings.
\end{abstract}

\section{Introduction}
 
The traditional understanding of machine learning suggests that models with zero training error tend to overfit, and explicit regularization is often necessary to achieve good generalization. Given the empirical success of deep learning models with zero training error \citep{ZBHRV:rethinking,NTS:real-inductive-bias} and the (re-)discovery of the ``double descent'' phenomenon \citep{reconcile:interpolation}, however, it has become clear that the textbook U-shaped learning curve is only part of a larger picture: it is possible for an overparameterized model with zero training loss to achieve low population error in a noisy setting. In an effort to understand how interpolation learning occurs, there has been much recent study of the testbed problem of linear regression with Gaussian features \citep[e.g.][]{bartlett2020benign, tsigler2020benign, BHX:two-models, hastie:surprises, JLL:basis-pursuit, junk-feats, muthukumar:interpolation, negrea:in-defense}. Significant progress has been made in this setting, including nearly-matching necessary and sufficient conditions for consistency of the minimal $\ell_2$ norm interpolator \citep{bartlett2020benign}.

Despite the fundamental role of uniform convergence in statistical learning theory, most of this line of work has used other techniques to analyze the particular minimal-norm interpolator.\footnote{\citet{negrea:in-defense} argue that \citet{bartlett2020benign}'s proof technique is fundamentally based on uniform convergence of a \emph{surrogate} predictor; \citet{YBM:random-feats-gap} study a closely related setting with a uniform convergence-type argument, but do not establish consistency. We discuss both papers in more detail in \cref{sec:euclidean}.}
Instead of directly analyzing the population error of a learning algorithm, a uniform convergence-type argument would control the worst-case generalization gap over a class of predictors containing the typical outputs of a learning rule.
Typically, this is done because for many algorithms -- unlike the minimal Euclidean norm interpolator -- it is difficult to exactly characterize the learned predictor, but we may be able to say e.g.\ that its norm is not too large.
Since uniform convergence does not tightly depend on a specific algorithm, the resulting analysis can highlight the key properties that lead to good generalization: it can give bounds not only for, say, the minimal-norm interpolator, but also for other interpolators with low norm \citep[e.g.][]{junk-feats}, increasing our confidence that low norm -- and not some other property the particular minimal-norm interpolator happens to have -- is key to generalization. In linear regression, practical training algorithms may not always find the exact minimal Euclidean norm solution, so it is also reassuring that all interpolators with sufficiently low Euclidean norm generalize.

\citet{NK:uniform}, however, raised significant questions about the applicability of typical uniform convergence arguments to certain high-dimensional regimes, similar to those seen in interpolation learning. Following their work, \citet{junk-feats, BL:failures, YBM:random-feats-gap, negrea:in-defense} all demonstrated the failure of forms of uniform convergence in various interpolation learning setups. To sidestep these negative results, \citet{junk-feats} suggested considering bounds which are uniform only over predictors \emph{with zero training error}.  This weaker notion of uniform convergence has been standard in analyses of ``realizable'' (noiseless) learning  %
at least since the work of \citet[Chapter 6.4]{vapnik82} and \citet{Valiant84}. 
\citeauthor{junk-feats} demonstrated that at least in one particular noisy setting, such uniform convergence is sufficient for showing consistency of the minimal $\ell_2$ norm interpolator, even though ``non-realizable'' uniform convergence arguments (those over predictors regardless of their training error) cannot succeed. It remains unknown, however, whether these types of arguments can apply to more general linear regression problems and more typical asymptotic regimes, particularly showing rates of convergence rather than just consistency.

In this work, we show for the first time that uniform convergence is indeed able to explain benign overfitting %
in general high-dimensional Gaussian linear regression problems. 
Similarly to how the standard analysis for learning with Lipschitz losses bounds generalization gaps through Rademacher complexity \citep[e.g.][]{understanding-ML},
our \cref{thm:main-generalization} (\cref{sec:generic}) establishes a finite-sample high probability bound on the uniform convergence of the error of interpolating predictors in a hypothesis class, in terms of its Gaussian width.
This is done through an application of the Gaussian Minimax Theorem; see the proof sketch in \cref{sec:sketches}.
Combined with an analysis of the norm of the minimal $\ell_2$ norm interpolator (\cref{thm:interpolator-euclidean} in \cref{sec:euclidean}), our bound recovers known consistency results \citep{bartlett2020benign}, as well as proving a conjectured upper bound for larger-norm interpolators \citep{junk-feats}.

In addition, since we do not restrict ourselves to Euclidean norm balls but instead consider interpolators in an arbitrary compact set, our results allows for a wide range of other applications. 
Our analysis leads to a natural extension of the consistency result and notions of effective rank of \citet{bartlett2020benign} for arbitrary norms (\cref{thm:min-norm} in \cref{sec:general-norm}).
As a demonstration of our general theory, in \cref{sec:basis-pursuit} we show novel consistency results for the minimal $\ell_1$ norm interpolator (basis pursuit) in particular settings, which we believe are the first results of their kind.

\section{Problem Formulation}
\paragraph{Notation.}
We use $\Norm\cdot_p$ for the $\ell_p$ norm, $\norm{x}_p = \left(\sum_i |x_i|^p\right)^{1/p}$. %
We always use $\max_{x \in S} f(x)$ to be $-\infty$ when $S$ is empty, and similarly $\min_{x \in S} f(x)$ to be $\infty$. 
We use standard $O(\cdot)$ notation, and $a \lesssim b$ for inequality up to an absolute constant.
For a positive semidefinite matrix $A$, the \emph{Mahalanobis (semi-)norm} is $\norm{x}_A^2 := \langle x, A x \rangle$.
For a matrix $A$ and set $S$,
$A S$ denotes the set $\{ A x : x \in S \}$.

\paragraph{Data model.} 
We assume that data $(X,Y)$ is generated as
\begin{equation} \label{eqn:model}
    Y = X w^* + \xi
    ,\qquad
    X_i \stackrel{iid}{\sim} N(0, \Sigma)
    ,\qquad
    \xi \sim N(0, \sigma^2 I_n)
,\end{equation}
where $X \in \R^{n \times d}$ has i.i.d.\ Gaussian rows $X_1,\ldots,X_n$, $d \ge n$, $w^*$ is arbitrary,
and $\xi$ is Gaussian and independent of $X$. 
Though our proof techniques crucially depend on $X_i$ being Gaussian, we can easily relax the assumption on the noise $\xi$ to only being sub-Gaussian; we assume Gaussian noise here for simplicity.
The \emph{empirical} and \emph{population loss} are defined as, respectively,
\[
    \hat{L}(w) = \frac{1}{n} \Norm{Y - X w}_2^2
    ,\quad
    L(w) = \E_{(x,y)} (y - \langle w, x \rangle)^2 = \sigma^2 + \Norm{w - w^*}_{\Sigma}^2
,\]
where in the expectation $y = \langle x, w^* \rangle + \xi_0$ with $x \sim N(0,\Sigma)$ independent of $\xi_0 \sim N(0,\sigma^2)$. For an arbitrary norm $\norm{\cdot}$, the minimal norm interpolator is $\hat{w} = \argmin_{\hat{L}(w) = 0} \norm w$.
For Euclidean norm specifically, the minimal norm interpolator can be written explicitly as $\hat{w} = X(XX^T)^{-1}Y$.
If there is more than one minimal norm interpolator, all of our guarantees will hold for any minimizer $\hat w$.
\paragraph{Speculative bound.} \citet{junk-feats} studied uniform convergence of low norm interpolators,
\begin{equation}\label{eq:supInt}
 \sup_{\norm w \le B, \, \hat L(w) = 0} L(w) - \hat L(w).
\end{equation}
Clearly, when $B \ge \norm{\hat w}$, this quantity upper-bounds the population risk of $\hat w$.  \citeauthor{junk-feats} evaluated the asymptotic limit of \eqref{eq:supInt} in one particular setting.
But they further speculated that a bound of the following form may hold more generally:
\begin{equation} \label{eq:spec-ub}
    \sup_{\norm{w}_2 \leq B, \, \hat{L}(w) = 0}
        L(w) - \hat{L}(w)
    \le \frac{B^2 \psi_n}{n}  + o(1)
    \tag{$\star$}
,\end{equation}
where\footnote{When $\norm{x}^2$ concentrates, we need $\psi_n$ to match its typical value. That is, $\psi_n$ might be a high probability bound on $\norm{x}^2$, or for (sub)Gaussian data, as in our case, $\psi_n=\E\norm{x}^2$.} $\psi_n \approx \norm{x}^2$ .  As discussed by \citeauthor{junk-feats}, a bound almost of this form is implied by results of \citet{optimistic-rates} for general data distributions, except that approach gives a large leading constant and logarithmic factors. To show consistency of benign overfitting, though, we need \eqref{eq:spec-ub} to hold without even a constant multiplicative factor. \citeauthor{junk-feats} ask whether and when this holds, speculating that it might in broad generality.

The goal of this paper is essentially to prove \eqref{eq:spec-ub}, at least for Gaussian data, and to use it to show consistency of the minimal norm interpolator. Our main result (\cref{thm:main-generalization}) can be thought of showing \eqref{eq:spec-ub} for Gaussian data with $\psi_n=\E\norm{x}^2$, as well as strengthening and significantly generalizing it.  Our result is more general, as it applies to general compact %
hypothesis sets beyond just the Euclidean ball as in \eqref{eq:spec-ub}.  But it also falls short of fully proving the speculative \eqref{eq:spec-ub} since our results are limited to Gaussian data, while there is no reason we are aware of to believe a tight uniform convergence guarantee of the form \eqref{eq:spec-ub} does not hold much more broadly.  We leave extending our results beyond the Gaussian case open.

\section{Generic Uniform Convergence Guarantee} \label{sec:generic}
To state our results,
we first need to introduce some key tools.

\begin{restatable}{defn}{gwidthrad} \label{def:gwidth-rad}
The \emph{Gaussian width} and the \emph{radius} of a set $S \subset \mathbb{R}^d$ are
\[ W(S) := \E_{H \sim N(0,I_d)} \sup_{s \in S} |\langle s, H \rangle| \quad \text{and} \quad  \rad(S) := \sup_{s \in S} \|s\|_2 .\]
\end{restatable}

The radius measures the size of a set in the Euclidean norm. The Gaussian width of a set $S$ can be interpreted as the number of dimensions that a random projection needs to approximately preserve the norms of points in $S$ \citep{ten-lectures, gordon1988}. These two complexity measures are connected by Gaussian concentration: Gaussian width is the expected value of the supremum of some Gaussian process, and the radius 
upper bounds the typical deviation of that supremum from its expected value.
\begin{restatable}[Covariance splitting]{defn}{covsplit}
Given a positive semidefinite matrix $\Sigma \in \R^{d \times d}$,
we write $\Sigma = \Sigma_1 \oplus \Sigma_2$ if
$\Sigma = \Sigma_1 + \Sigma_2$,
each matrix is positive semidefinite,
and their spans are orthogonal.
\end{restatable}
Note that $\Sigma = \Sigma_1 \oplus \Sigma_2$ effectively splits the eigenvectors of $\Sigma$ into two disjoint parts.

We can now state our generic bound.
\Cref{sec:sketches} sketches the proof;
all full proofs are in the appendix.

\begin{restatable}[Main generalization bound]{theorem}{maingenbound} \label{thm:main-generalization}
There exists an absolute constant $C_1 \leq 66$ such that the following is true. Under the model assumptions in \eqref{eqn:model}, let $\cK$ be an arbitrary compact set, and
take any covariance splitting $\Sigma = \Sigma_1 \oplus \Sigma_2$.
Fixing $\delta \le 1/4$, let $\beta = C_1 \left(\sqrt{\frac{\log(1/\delta)}{n}} + \sqrt{\frac{\rank(\Sigma_1) }{n}} \right)$.
If $n$ is large enough that $\beta \leq 1$, then the following holds with probability at least $1 - \delta$:
\begin{equation*}
    \sup_{w \in \cK, \hat L(w) = 0} L(w)
    \le \frac{1 + \beta}{n}\left[
        W(\Sigma_2^{1/2} \cK)
        + \left( \rad(\Sigma_2^{1/2} \cK) + \Norm{w^*}_{\Sigma_2} \right) \sqrt{2\log \left(\frac{32}{\delta}\right)}
    \, \right]^2.
\end{equation*}
\end{restatable}

In our applications, we consider $\mathcal{K} = \{ w \in \R^d : \| w \| \leq B \}$ for an arbitrary norm, with $B$ based on a high-probability upper bound for $\norm{\hat{w}}$.
Depending on the application, the rank of $\Sigma_1$ will be either constant or $o(n)$, so that $\beta \to 0$.
The term $\norm{w^*}_{\Sigma_2}$ generally does not scale with $n$ and hence is often negligible.
As hinted earlier, we can think of the Gaussian width term and the radius term as bias and variance, respectively.
To achieve consistency, we can expect that the Gaussian width should scale as $\sigma \sqrt{n}$. This agrees with the intuition that we need increasing norm to memorize noise when the model is not realizable. The radius term requires some care in our applications, but can be handled by the covariance splitting technique. 
As part of the analysis in the following sections, we will rigorously show in many settings that the dominant term in the upper bound is the Gaussian width. In these cases, our 
upper bound is roughly $W(\Sigma_2^{1/2} \cK)^2/n$, which can be viewed as the ratio between the (probabilistic) dimension of our hypothesis class and sample size.  
We will also analyze how large $\cK$ must be to contain any interpolators, allowing us to find consistency results.

\section{Application: Euclidean Norm Ball} 
\label{sec:euclidean}

It can be easily seen that the Gaussian width of a Euclidean norm ball reduces nicely to the product of the norm of our predictor with the typical norm of $x$: if $\mathcal{K} = \{ w \in \R^d : \| w \|_2 \leq B \}$, then
\begin{equation} \label{eqn:width-trace}
W(\Sigma^{1/2} \cK) = B \cdot \E_{H \sim N(0, I_d)} \| \Sigma^{1/2} H \|_2 \leq \sqrt{B^2 \E \norm{x}_2^2} 
.\end{equation}
Therefore, it is plausible that \eqref{eq:spec-ub} holds with $\psi_n = \E \| x \|_2^2 = \Tr(\Sigma)$. \Cref{fig:example} illustrates this generalization bound in two simple examples, motivated by \citet{hastie:surprises,junk-feats}. Indeed, an application of our main theorem proves that this is exactly the case for Gaussian data.

\begin{restatable}[Proof of the speculative bound \eqref{eq:spec-ub} for Gaussian data]{corollary}{specbound}
\label{corr:spec-bound}
Fix any $\delta \leq 1/4$. Under the model assumptions in \eqref{eqn:model} with $B \geq \|w^*\|_2$ and $n \gtrsim \log(1/\delta) $, for some $\gamma \lesssim \sqrt[4]{\log(1/\delta)/n}$, it holds with probability at least $1 - \delta$ that
\begin{equation}
    \sup_{\|w\|_2 \le B, \hat L(w) = 0} L(w) \le (1 + \gamma) \frac{B^2\Tr (\Sigma)}{n}.
\end{equation}
\end{restatable}

\begin{figure}
    \centering
    \begin{subfigure}{.49\textwidth}
    \centering
    \includegraphics[scale=0.38]{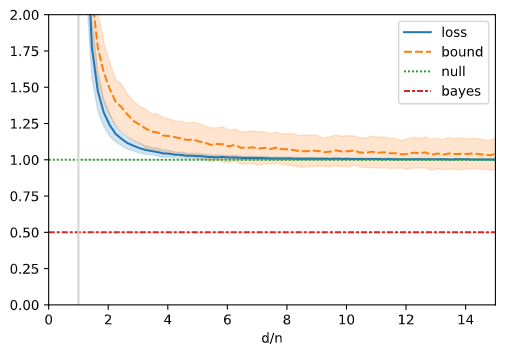}
    \caption{$\lambda = 1$ (isotropic)}
    \end{subfigure}    
    \begin{subfigure}{.49\textwidth}
    \centering
    \includegraphics[scale=0.51]{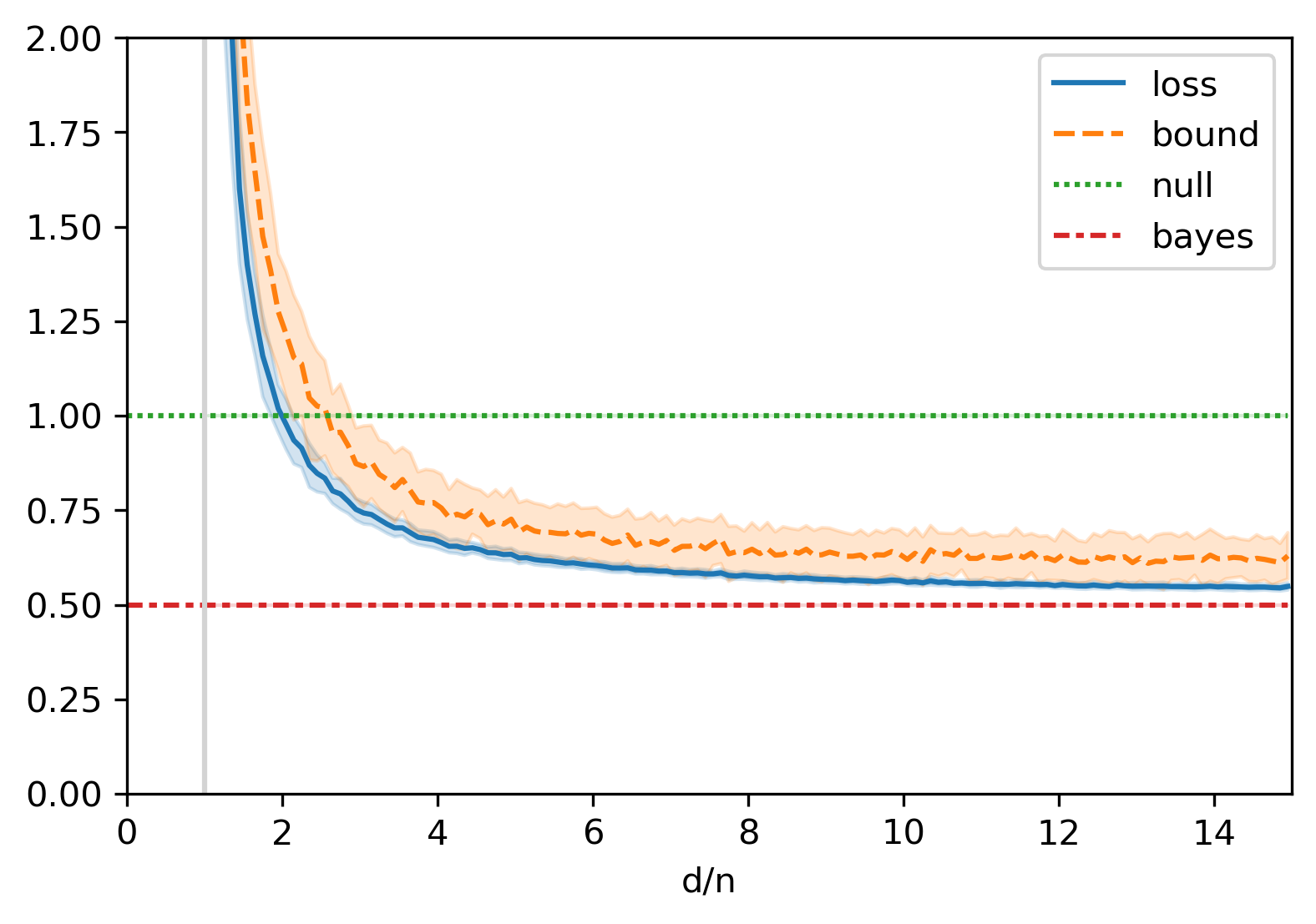}
    \caption{$\lambda = 0.1$}
    \end{subfigure}
    
    \caption{Illustration of our generalization bound when $\Sigma = \begin{bmatrix} 1 & 0 \\ 0 & \lambda^2 I_{d - 1} \end{bmatrix}$, $n = 200$, $\sigma^2 = 1/2$, $w^* = (1/\sqrt{2},0,\ldots,0)$, %
    and $d$ is varied (x-axis). Averages (curve) and standard deviations (error bars) are estimated from 400 trials for each value of $d$. 
    Here the curve marked ``loss'' corresponds to $L(\hat w)$ for the minimum Euclidean norm interpolator $\hat w$, ``bound'' to ${\|\hat w\|_2^2 \Tr \Sigma}/{n} = \E {\|\hat w\|_2^2 (1 + \lambda^2 (d - 1))}/{n}$ which is an asymptotic bound on $L(\hat w)$ due to \cref{corr:spec-bound}, ``null'' is the loss $L(0) = 1$ of the zero estimator, and ``bayes'' is the Bayes-optimal error $L(w^*) = \sigma^2 = 1/2$.  The vertical line is $d/n = 1$, the location of the double-descent peak; for $d/n < 1$ there are almost surely no interpolators.}
    \label{fig:example}
\end{figure}

The above bound is clean and simple, but only proves a sub-optimal rate of $n^{-1/4}$.
This is because the choice of covariance split used in the proof of \cref{corr:spec-bound} uses no information about the particular structure of $\Sigma$. This bound can also be slightly loose in situations where the eigenvalues of $\Sigma$ decay rapidly, in which case $\Tr(\Sigma)$ can be replaced by a smaller quantity. We next state a more precise bound on the generalization error, which requires introducing the following notions of {effective rank}. 

\begin{restatable}[\cite{bartlett2020benign}]{defn}{euclideffrank} \label{def:ranks-l2}
The \emph{effective ranks} of a covariance matrix $\Sigma$ are
\[
r(\Sigma) = \frac{\Tr(\Sigma)}{\norm{\Sigma}_\OP} \quad \text{and} \quad R(\Sigma) = \frac{\Tr(\Sigma)^2}{\Tr(\Sigma^2)}
.\]
\end{restatable}

The $r(\Sigma)$ rank can roughly
be understood as the squared ratio between the Gaussian width and radius in our previous bound.
It is related to the concentration of the $\ell_2$ norm of a Gaussian vector with covariance $\Sigma$. In fact, both definitions of effective ranks can be derived by applying Bernstein's inequality to ${\| x \|^2}/{\E \| x \|^2}$. 
We will only need $r(\Sigma)$ in the generalization bound below, but we will show in \cref{thm:interpolator-euclidean} that $R(\Sigma)$ can be used to control the norm of the minimal norm interpolator $\hat{w}$. 

\begin{restatable}{corollary}{euclidgen} \label{corr:euclidean-generalization}
There exists an absolute constant $C_1 \leq 66$ such that the following is true. Under \eqref{eqn:model}, pick any split $\Sigma = \Sigma_1 \oplus \Sigma_2$, fix $\delta \le 1/4$, and let $\gamma = C_1 \left(\sqrt{\frac{\log(1/\delta)}{r(\Sigma_2)}} + \sqrt{\frac{\log(1/\delta)}{n}} + \sqrt{\frac{\rank(\Sigma_1) }{n}} \right)$. If $B \geq \| w^* \|_2$ and $n$ is large enough that $\gamma \leq 1$, the following holds with probability at least $1 - \delta$:
\begin{equation}
    \sup_{\|w\|_2 \le B, \hat L(w) = 0} L(w)
 \le (1 + \gamma) \frac{B^2 \Tr(\Sigma_2)}{n}.
\end{equation}
\end{restatable}

In order to use \cref{corr:spec-bound} or \labelcref{corr:euclidean-generalization} to prove consistency, we need a high-probability bound for $\norm{\hat{w}}_2$, the norm of the minimal norm interpolator, so that $B$ will be large enough to contain any interpolators.
\Cref{thm:interpolator-euclidean} gives exactly such a bound,
showing that if the effective ranks $R(\Sigma_2)$ and $r(\Sigma_2)$ are large, then we can construct an interpolator with Euclidean norm nearly $\|w^*\|_2 + \sigma\sqrt{n/\Tr(\Sigma_2)}$. 

\begin{restatable}[Euclidean norm bound; special case of \cref{thm:interpolator-general}]{theorem}{euclidnormbound} \label{thm:interpolator-euclidean}
Fix any $\delta \leq 1/4$. Under the model assumptions in \eqref{eqn:model} with any choice of covariance splitting $\Sigma = \Sigma_1 \oplus \Sigma_2$, there exists some $\epsilon \lesssim  \sqrt{\frac{\log(1/\delta)}{r(\Sigma_2)}} + \sqrt{\frac{\log(1/\delta)}{n}} + \frac{n \log (1/\delta)}{R(\Sigma_2)} $ such that the following is true. If $n$ and the effective ranks are such that $\epsilon \leq 1$ and $ R(\Sigma_2) \gtrsim \log (1/\delta)^2$, then with probability at least $1-\delta$, it holds that
\begin{equation}
    \|\hat{w}\|_2 \le \|w^*\|_2 + (1+\epsilon)^{1/2} \, \sigma \sqrt{\frac{n}{\Tr (\Sigma_2)}}.
\end{equation}
\end{restatable}

Plugging in estimates of $\| \hat{w} \|$ to our scale-sensitive bound \cref{corr:euclidean-generalization}, we obtain a population loss guarantee for $\hat{w}$ in terms of effective ranks.

\begin{restatable}[Benign overfitting]{theorem}{minnormeuclid} \label{corr:min-norm-euclidean}
Fix any $\delta \leq 1/2$. Under the model assumptions in \eqref{eqn:model} with any covariance splitting $\Sigma = \Sigma_1 \oplus \Sigma_2$, let $\gamma$ and $\epsilon$ be as defined in \cref{corr:euclidean-generalization,thm:interpolator-euclidean}.
Suppose that $n$ and the effective ranks are such that $ R(\Sigma_2) \gtrsim \log (1/\delta)^2$ and $\gamma, \epsilon \leq 1$. Then, with probability at least $1-\delta$,
\begin{equation}\label{eq:min-norm-euclidean}
    L(\hat{w})
    \leq (1+\gamma)(1+\epsilon) \left( \sigma + \| w^* \|_2 \sqrt{\frac{\Tr (\Sigma_2)}{n}} \right)^2
.\end{equation}
\end{restatable}

From \eqref{eq:min-norm-euclidean},
we can see that to ensure consistency,
i.e.\ $L(\hat w) \to \sigma^2$,
it is enough that
$\gamma \to 0$, $\epsilon \to 0$,
and $\| w^* \|_2 \sqrt{\frac{\Tr (\Sigma_2)}{n}} \to 0$.
Recalling the definitions of the various quantities
and using that $r(\Sigma_2)^2 \ge R(\Sigma_2)$ \citep[Lemma 5]{bartlett2020benign},
we arrive at the following conditions.

\paragraph{Sufficient conditions for consistency of $\hat w$.}
As $n \to \infty$,
$L(\hat{w})$ converges in probability to $\sigma^2$ if there exists a sequence of covariance splits $\Sigma = \Sigma_1 \oplus \Sigma_2$ such that
\begin{equation}
    \frac{\rank(\Sigma_1)}{n} \to 0
    ,\qquad
    \norm{w^*}_2 \sqrt{\frac{\Tr(\Sigma_2)}{n}} \to 0
    ,\qquad
    \frac{n}{R(\Sigma_2)} \to 0
.\end{equation}

\paragraph{Relationship to \citet{bartlett2020benign}.} 

Our set of sufficient conditions above subsumes and is slightly more general than the conditions of \citet{bartlett2020benign}. There are two differences:
\begin{enumerate}[noitemsep,topsep=0pt]
    \item They choose the covariance split specifically to minimize $\rank(\Sigma_1)$ such that $r(\Sigma_2) \gtrsim n$.
    \item Their version of %
    the second condition replaces $\Tr(\Sigma_2)$ by the larger term $\Tr(\Sigma)$.
\end{enumerate}
From the perspective of showing $L(\hat{w}) \to \sigma^2$, the first difference is immaterial: if there exists a choice of split that satisfies our conditions, it can be shown that there exists a (possibly different) split which will also satisfy $r(\Sigma_2) \gtrsim n$ (see \cref{apdx:equivalent}).
The second point is a genuine improvement over the consistency result of \citet{bartlett2020benign} when $\Sigma$ has a few very large eigenvalues; this improvement has also been implicitly obtained by \citet{tsigler2020benign}.

Regarding the rate of convergence, our additional $r(\Sigma_2)^{-1/2}$ term and the dependence on $\sqrt{\rank(\Sigma_1)/n}$ instead of $\rank(\Sigma_1)/n$ is slightly worse than that of \citet{bartlett2020benign}, but our bound can be applied for a smaller value of $\rank(\Sigma_1)$ and is better in the $\| w^* \|_2 \sqrt{\Tr (\Sigma_2)/n}$ term.
We believe these differences are minimal in most cases, and not so important for our primary goal to showcase the power of uniform convergence.

\paragraph{Relationship to \citet{negrea:in-defense}.} The consistency result of \citet{bartlett2020benign} can also be recovered with a uniform convergence-based argument \citep{negrea:in-defense}.
Instead of considering uniform convergence over a norm ball, \citeauthor{negrea:in-defense} applied uniform convergence to a surrogate predictor, and separately showed that the minimal-norm interpolator has risk close to the surrogate (and, indeed, argue that this was fundamentally the proof strategy of \citet{bartlett2020benign} all along). Their analysis reveals an interesting connection between realizability and interpolation learning, but it does not highlight that low norm is key to good generalization, nor does it predict the worst-case error for other low-norm interpolators.

\paragraph{Relationship to \citet{BL:failures}.} 
\citeauthor{BL:failures} recently showed that it is impossible to find a tight excess risk bound that only depends on the learned predictor and sample size. 
This does not, however, contradict our results. A closer look at the construction of their lower bound reveals that the excess risk bounds being ruled out cannot depend on either the training error $\hat L$ or the population noise level $\sigma^2$. The former is crucial: their considered class of bounds cannot incorporate the knowledge that the training error is small, which is the defining property of uniform convergence of interpolators. The latter point is also important; they consider excess risk ($L - \sigma^2$), but \eqref{eq:spec-ub} and our bounds are about the generalization gap ($L - \hat L$).

\paragraph{Relationship to \citet{YBM:random-feats-gap}.} 
\citeauthor{YBM:random-feats-gap} give expressions for the asymptotic generalization error of predictors in a norm ball, in a random feature model.
Their model is not directly comparable to ours (their labels are effectively a nonlinear function of their non-Gaussian random features), but they similarly showed that uniform convergence of interpolators can lead to a non-vacuous bound. It is unclear, though, whether uniform convergence of low-norm interpolators can yield consistency in their model: they only study sets of the form $\{ \norm w \le \alpha \norm{\hat w} \}$ with $\alpha > 1$ a constant, where we would expect a loss of $\alpha^2 \sigma^2$ -- i.e.\ would not expect consistency. They also rely on numerical methods to compare their (quite complicated) analytic expressions. It remains possible that the gap between uniform convergence of interpolators and the Bayes risk vanishes in their setting as $\alpha$ approaches 1.

\section{General Norm Ball} \label{sec:general-norm}
All the results on the Euclidean setting are special cases of the following results for arbitrary norms.
It is worth keeping in mind that the Euclidean norm will still play a role in these analyses, via the Gaussian width and radius appearing in \cref{thm:main-generalization} and the $\ell_2$ projection $P$ appearing in \cref{thm:interpolator-general}.

\begin{restatable}{defn}{dualnorm}
The \emph{dual norm} of a norm $\norm\cdot$ on $\mathbb{R}^d$ is
$ \|u\|_* := \max_{\|v\| = 1} \langle v, u \rangle $,
and the set of all its sub-gradients with respect to $u$ is $ \partial \|u\|_* = \{v : \|v\| = 1, \langle v, u \rangle = \|u\|_* \}$. 
\end{restatable}

The Euclidean norm's dual is itself;
for it and many other norms, $\partial \norm{u}_*$ is a singleton set.
Using these notions, we will now give versions of the effective ranks appropriate for generic norm balls.

\begin{restatable}{defn}{geneffrank} \label{def:ranks-generic}
The effective $\norm\cdot$-ranks of a covariance matrix $\Sigma$ are given as follows.
Let $H \sim N(0, I_d)$,
and define $v^* = \argmin_{v \in \partial \|\Sigma^{1/2} H\|_*} \norm{v}_{\Sigma}$.
Then
\[
 r_{\norm\cdot}(\Sigma) = \left( \frac{\E \norm{\Sigma^{1/2} H}_*}{\sup_{\norm w \le 1} \norm{w}_{\Sigma}} \right)^2
\quad \text{and} \quad
R_{\norm\cdot}(\Sigma) = \left( \frac{\E \norm{\Sigma^{1/2} H}_* }{\E \norm{v^*}_{\Sigma}} \right)^2
.\]
\end{restatable}

The first effective $\norm\cdot$-rank is the squared ratio of Gaussian width to the radius of the set $\Sigma^{1/2} \cK$, where $\cK$ is a norm ball $\{ w : \norm w \le B \}$;
the importance of this ratio should be clear from \cref{thm:main-generalization}.
The Gaussian width is given by $W(\Sigma^{\frac12} \cK) = B \E \norm x_*$,
while the radius can be written $\sup_{\norm{w} \le B} \norm{w}_\Sigma$
so that the factors of $B$ cancel. 

The choice of $R_{\norm\cdot}$ arises naturally from our bound on $\norm{\hat w}$ in \cref{thm:interpolator-general} below. Large effective rank of $\Sigma_2$ means the sub-gradient $v^*$ of $\|\Sigma_2^{1/2} H\|_*$ is small in the $\Norm\cdot_{\Sigma_2}$ norm. This is, in fact, closely related to the existence of low-norm interpolators. First, note that $\Sigma_2^{1/2} H$ corresponds to the small-eigenvalue components of the covariate vector. For $v^*$ to be a sub-gradient means that moving the weight vector $w$ in the direction of $v^*$ is very effective at changing the prediction $\langle w, X \rangle$; having small $\Norm\cdot_{\Sigma_2}$ norm means that moving in this direction has a very small effect on the population loss $L(w)$. Together, this means the sub-gradient will be a good direction for benignly overfitting the noise.

\begin{remark}[Definitions of effective ranks] Using $\norm\cdot_2$ in \cref{def:ranks-generic} yields slightly different effective ranks than those of \cref{def:ranks-l2}, but the difference is small and asymptotically negligible.
Both $r$ and $R$ use $\E \norm x^2$ in their numerators,
while $r_{\norm\cdot_2}$ and $R_{\norm\cdot_2}$ use $(\E \norm x)^2$.
The denominators of $r$ and $r_{\norm\cdot_2}$ agree; \cref{lem:rank-fact}, in \cref{apdx:norm-bounds:euclidean}, shows that $r(\Sigma) - 1 \le r_{\norm\cdot_2}(\Sigma) \le r(\Sigma)$.
The denominator of $R_{\norm\cdot_2}$ uses the sole sub-gradient
$v^* = \Sigma^{1/2} H / \Norm{\Sigma^{1/2} H}_2$;
the denominator is then $\Norm{\Sigma^{1/2} H}_\Sigma^2 / \Norm{\Sigma^{1/2} H}_2^2 \approx \Tr(\Sigma^2) / \Tr(\Sigma)$,
giving that $R_{\norm\cdot_2}(\Sigma) \approx \Tr(\Sigma)^2 / \Tr(\Sigma^2) = R(\Sigma)$. \Cref{eq:big-r-defs-close}, in \cref{apdx:norm-bounds:euclidean}, shows that $R_{\norm\cdot_2}(\Sigma) \geq cR(\Sigma)$ for some $c$ that converges to 1 as $r(\Sigma) \to \infty$, as is required by the consistency conditions. The other direction, $R(\Sigma) \geq c' R_{\norm\cdot_2}(\Sigma)$, also holds with $c' \to 1$ as $r(\Sigma^2) \to \infty$. It would be possible (and probably even more natural with our analysis) to state the consistency conditions in \cref{sec:euclidean} in terms of $r_{\norm\cdot_2}$ and $R_{\norm\cdot_2}$; we used $r$ and $R$ mainly to allow direct comparison to \citet{bartlett2020benign}.
\end{remark}

Using the general notion of effective ranks,
we can find an analogue of \cref{corr:euclidean-generalization} for general norms. 

\begin{restatable}{corollary}{maingen} \label{corr:main-generalization}
There exists an absolute constant $C_1 \leq 66$ such that the following is true.
Under the model assumptions in \eqref{eqn:model}, take any covariance splitting $\Sigma = \Sigma_1 \oplus \Sigma_2$ and let $\norm\cdot$ be an arbitrary norm.
Fixing $\delta \le 1/4$,
let $\gamma = C_1 \left(\sqrt{\frac{\log(1/\delta)}{r_{\| \cdot \|}(\Sigma_2)}} + \sqrt{\frac{\log(1/\delta)}{n}} + \sqrt{\frac{\rank(\Sigma_1) }{n}} \right)$. If $B \geq \| w^* \|$ and $n$ is large enough that $\gamma \leq 1$, then the following holds with probability at least $1 - \delta$:
\begin{equation}
    \sup_{\|w\| \le B, \hat L(w) = 0} L(w)
 \le (1+\gamma) \frac{\left( B \cdot \E \Norm{ \Sigma^{1/2}_2 H }_* \right)^2}{n}.
\end{equation}
\end{restatable}

As in the Euclidean special case, we still need a bound on the norm of $\hat w = \argmin_{\hat L(w) = 0} \norm{w}$ to use this result to study the consistency of $\hat w$.
This leads us to the second main technical result of this paper,
which essentially says that if the effective ranks $R_{\norm\cdot}(\Sigma_2)$ and $r_{\norm\cdot}(\Sigma_2)$ are sufficiently large, then there exists an interpolator with norm $\norm{w^*} + \sigma\sqrt{n} / \E\Norm{ \Sigma_2^{1/2} g}_*$. 

\begin{restatable}[General norm bound]{theorem}{gennormbound} \label{thm:interpolator-general}
There exists an absolute constant $C_2 \leq 64$ such that the following is true.
Under the model assumptions in \eqref{eqn:model} with any covariance split $\Sigma = \Sigma_1 \oplus \Sigma_2$, let $\norm\cdot$ be an arbitrary norm, and fix $\delta \le 1/4$.
Denote the $\ell_2$ orthogonal projection matrix onto the space spanned by $\Sigma_2$ as $P$.
Let $H \sim N(0,I_d)$,
and let $v^* = \argmin_{v \in \partial \|\Sigma^{1/2}_2 H\|_*} \norm{v}_{\Sigma_2}$.
Suppose that there exist $\epsilon_1, \epsilon_2 \geq 0$ such that with probability at least $1-\delta/4$
\begin{equation} 
    \norm{v^*}_{\Sigma_2} \leq (1+\epsilon_1) \E \norm{v^*}_{\Sigma_2}
\qquad\text{and}\qquad
    \norm{P v^*}^2 \leq 1 + \epsilon_2
;\end{equation}
let $\epsilon = C_2 \left( \sqrt{\frac{\log(1/\delta)}{r_{\norm\cdot}(\Sigma_2)}} + \sqrt{\frac{\log(1/\delta)}{n}} + (1+\epsilon_1)^2 \frac{n}{R_{\norm\cdot}(\Sigma_2)} + \epsilon_2 \right)$. Then if $n$ and the effective ranks are large enough that $\epsilon \leq 1$, with probability at least $1-\delta$, it holds that
\begin{equation}
    \norm{\hat w} \leq \norm{w^*} + (1+\epsilon)^{1/2} \, \sigma \frac{\sqrt{n}}{\E \| \Sigma^{1/2}_2 H \|_*}.
\end{equation}
\end{restatable}

For a specific choice of norm $\norm\cdot$, we can verify that $\|v^*\|_{\Sigma_2}$ is small. In the Euclidean case, for example, this is done by \eqref{eqn:v*-norm} in \cref{apdx:norm-bounds:euclidean}; in our basis pursuit application to come, this is done by \eqref{eqn:v*-norm-bp}.
The term $\epsilon_2$ measures the cost of using a projected version of the subgradient; in most of our applications, we can take $\epsilon_2 = 0$. Recalling that $\Norm{v^*} = 1$, this is obviously true with the Euclidean norm for any $\Sigma_2$.
More generally, if $\Sigma$ is diagonal, then it is natural to only consider covariance splits $\Sigma = \Sigma_1 \oplus \Sigma_2$ such that $\Sigma_2$ is diagonal. Then, when $\norm\cdot$ is the $\ell_1$ norm (or $\ell_p$ norms more generally), it can be easily seen that $Pv^* = v^*$ and so $\| P v^* \| = \| v^* \| = 1$.

Straightforwardly combining \cref{corr:main-generalization,thm:interpolator-general} yields the following theorem, which gives guarantees for minimal-norm interpolators in terms of effective rank conditions. Just as in the Euclidean case, we can extract from this result a simple set of sufficient conditions for consistency of the minimal norm interpolator.

\begin{restatable}[Benign overfitting with general norm]{theorem}{minnorm} \label{thm:min-norm}
Fix any $\delta \leq 1/2$. Under the model assumptions in \eqref{eqn:model}, let $\norm\cdot$ be an arbitrary norm and pick a covariance split $\Sigma = \Sigma_1 \oplus \Sigma_2$. Suppose that $n$ and the effective ranks are sufficiently large such that $\gamma, \epsilon \leq 1$ with the same choice of $\gamma$ and $\epsilon$ as in \cref{corr:main-generalization} and \cref{thm:interpolator-general}. Then, with probability at least $1-\delta$,
\begin{equation}
    L(\hat{w})
    \leq (1 + \gamma)(1+\epsilon) \left( \sigma + \| w^* \| \frac{\E \|\Sigma_2^{1/2} H\|_*}{\sqrt{n}} \right)^2.
\end{equation}
\end{restatable}

\paragraph{Sufficient conditions for consistency of $\hat w$.}
As $n \to \infty$,
$L(\hat{w})$ converges in probability to $\sigma^2$ if there exists a sequence of covariance splits $\Sigma = \Sigma_1 \oplus \Sigma_2$ such that
\begin{equation} \label{eq:suffcon:general}
    \frac{\rank(\Sigma_1)}{n} \to 0
    ,\qquad
    \frac{\norm{w^*} \E \|\Sigma^{1/2}_2 H\|_*}{\sqrt{n}} \to 0
    ,\qquad
    \frac{1}{r_{\|\cdot\|}(\Sigma_2)} \to 0
    ,\qquad
    \frac{n}{R_{\|\cdot\|}(\Sigma_2)} \to 0,
\end{equation}
and, with the same definition of $P$ and $v^*$ as in \cref{thm:interpolator-general}, it holds for any $\eta > 0$ that
\begin{equation} \label{eqn:contraction}
    \Pr(\| P v^* \|^2 > 1 + \eta) \to 0.
\end{equation}

As we see, the conditions for a minimal norm interpolator to succeed with a general norm generalize those from the Euclidean setting in a natural way. As discussed above, \eqref{eqn:contraction} is always satisfied for the Euclidean norm. The only remaining notable difference from the Euclidean setting is that we have two large effective dimension conditions on $\Sigma_2$ instead of a single one; in the Euclidean case, the condition on $R$ implies the condition on $r$.

\section{Application: \texorpdfstring{$\ell_1$}{L1} Norm Balls for Basis Pursuit} \label{sec:basis-pursuit}
The theory for the minimal $\ell_1$ norm interpolator, $\hat w_\BP \in \arg\min_{\hat L(w) = 0} \norm{w}_1$ -- also known as \emph{basis pursuit} \citep{CDS:basis-pursuit} -- is much less developed than that of the minimal $\ell_2$ norm interpolator.
In this section, we illustrate the consequences of our general theory for basis pursuit. 
Full statements and proofs of results in this section are given in \cref{apdx:bp}.

The dual of the $\ell_1$ norm is the $\ell_{\infty}$ norm $\norm{u}_\infty = \max_i \lvert u_i \rvert$, and 
$\partial \|u\|_\infty$ is the convex hull of  $\{ \operatorname{sign}(u_i) \, e_i : i \in \arg\max |u_i| \}$.
From the definition of sub-gradient, we observe that
\begin{equation} \label{eqn:bp-fact1}
    \min_{v \in \partial \|\Sigma^{1/2} g\|_\infty}\|v\|_{\Sigma} \le \max_{i \in [d]} \, \|e_i\|_{\Sigma} = \sqrt{\max_i \, \Sigma_{ii}} 
.\end{equation} 
Furthermore, by convexity we have
\begin{equation} \label{eqn:bp-fact2}
    \max_{\|w\|_1 \le 1} \|w\|_{\Sigma} =  \sqrt{\max_{i} \, \langle e_i, \Sigma e_i \rangle} = \sqrt{\max_i \, \Sigma_{ii}}
\end{equation}
and so $r_{\| \cdot \|_1}(\Sigma) = \frac{\left(\E \| \Sigma^{1/2} g \|_{\infty}\right)^2}{\max_i \Sigma_{ii}} \leq R_{\| \cdot \|_1}(\Sigma)$. Therefore, we can use a single notion of effective rank. For simplicity, we denote $r_1(\Sigma) = r_{\| \cdot \|_1}(\Sigma)$. Combining this with \eqref{eq:suffcon:general} and the previous discussion of \eqref{eqn:contraction}, we obtain the following sufficient conditions for consistency of basis pursuit.

\paragraph{Sufficient conditions for consistency of $\hat w_\BP$.}
As $n \to \infty$,
$L(\hat{w})$ converges to $\sigma^2$ in probability if there exists a sequence of covariance splits $\Sigma = \Sigma_1 \oplus \Sigma_2$ such that $\Sigma_2$ is diagonal and 
\begin{equation}
    \frac{\rank(\Sigma_1)}{n} \to 0
    ,\qquad
    \frac{\norm{w^*}_1 \E \Norm{\Sigma_2^{1/2} H}_\infty}{\sqrt n} \to 0
    ,\qquad
    \frac{n}{r_1(\Sigma_2)} \to 0
.\end{equation}

\paragraph{Application: Junk features.}
We now consider the behavior of basis pursuit in a junk feature model similar to that of \citet{junk-feats}. Suppose that
$ \Sigma = 
\begin{bmatrix} 
\Sigma_s & 0 \\ 
0 & \frac{\lambda_n}{\log(d)} I_{d} 
\end{bmatrix}
$,
where $\Sigma_s$ is a fixed matrix and $\|w^*\|_1$ is fixed. Quite naturally, we choose the covariance splitting $\Sigma_1 = \begin{bmatrix} 
\Sigma_s & 0 \\ 
0 & 0
\end{bmatrix}$, which has constant rank so that the first sufficient condition is immediately satisfied. 

By standard results on the maximum of independent Gaussian variables \citep[e.g.][]{vershynin2018high}, it is routine to check that
\begin{equation}
    \frac{\E \|\Sigma_2^{1/2} H\|_{\infty}}{\sqrt{n}} = \Theta\left( \sqrt{\frac{\lambda_n}{n}}\right) \quad\text{and}\quad  r_1(\Sigma_2) = \Theta\left( \log(d) \right).
\end{equation}
Therefore, basis pursuit will be consistent provided that $\lambda_n = o(n)$ and $d = e^{\omega(n)}$. To the best of our knowledge, this is the first time that basis pursuit has been shown to give consistent predictions in any setting with Gaussian covariates and $\sigma > 0$. Although we show consistency, the dimension must be quite high, and the rate of convergence depends on ${n}/{\log(d)}$ and $1/\sqrt{\log(d)}$.

\paragraph{Application: Isotropic features.} As in the Euclidean case, we generally do not expect basis pursuit to be consistent when $\Sigma = I_d$ and $w^* \ne 0$. %
However, we can expect its risk to approach the null risk $\sigma^2 + \| w^* \|^2$ if $d = e^{\omega(n)}$; we will show this using uniform convergence (without covariance splitting). 

A direct application of \cref{thm:min-norm} is not enough because the $\| w^* \|_1 \sqrt{{\log(d)}/{n}}$ term diverges, but we can remove the dependence on $\sqrt{\log(d)/n}$ with a better norm bound.
Let $S$ be the support of $w^*$ and denote $X_S$ as the matrix formed by selecting the columns of $X$ in $S$.
The key observation is that we can rewrite our model as $Y = X_{S^\mathsf{c}} \, 0 +  (X_S w_S^* + \xi)$, which corresponds to the case when $w^* = 0$ and the Bayes risk is $\sigma^2 + \| w^* \|_2^2$.
If we interpolate using only the features in $S^\mathsf{c}$, the minimal norm will be approximately upper bounded by $\sqrt{\sigma^2 + \| w^* \|_2^2} \frac{\sqrt{n}}{\E \| H\|_*} $ as long as $d-|S| = e^{\omega(n)}$, by \cref{thm:interpolator-general}.
This implies the original model $\| \hat{w}_\BP \|_1$ can also be upper bounded by the same quantity with high probability. Plugging the norm estimate in to \cref{corr:main-generalization} yields a risk bound of $\sigma^2 + \| w^* \|_2^2$. 

\paragraph{Relationship to previous works.} 
Both \citet{JLL:basis-pursuit} and \citet{chinot2021robustness} study the minimal $\ell_1$ norm interpolator in the isotropic setting. They consider a more realistic scaling where $\log(d)/n$ is not large and the target is not the null risk. The best bound of \citet{JLL:basis-pursuit}, their Corollary 3, is $L(\hat w_\BP) \le \sigma^2(2 + 32\sqrt{14}\sqrt{s})^2$, where $s$ is the ground truth sparsity. Note that even when $w^* = 0$, this bound does not show consistency. Similarly, \citet{chinot2021robustness} establish sufficient conditions for $L(\hat w_\BP) = O(\sigma^2)$, which is nontrivial but also does not show consistency for any $\sigma > 0$; see also the work of \citet{chinot2020robustness} for a similar result in the Euclidean setting.  In contrast, the constants in our result are tight enough to show $L(\hat w_\BP) \to \sigma^2$ in the isotropic setting when $w^* = 0$ and in the junk feature setting when $\| w^*\|_1$ is bounded.

Like our work, the results of \citet{chinot2021robustness} generalize to arbitrary norms; %
they also consider a larger class of anti-concentrated covariate distributions than just Gaussians, as in the work of \citet{mendelson2014learning}.
If $\sigma = 0$ and $w^* \in \cK$ (i.e. the model is well-specified and noiseless), their work as well as that of \citet{mendelson2014learning} can recover generalization bounds similar to our \cref{corr:main-generalization}, but with a large leading constant.

\section{Proof Sketches} \label{sec:sketches}

A key ingredient in our analysis is a celebrated result from Gaussian process theory known as the Gaussian Minmax Theorem (GMT) \citep{gordon1985some,thrampoulidis2015regularized}. Since the seminal work of \citet{rudelson2008sparse}, the GMT has seen numerous applications to problems in statistics, machine learning, and signal processing \citep[e.g.][]{stojnic2013framework,deng2019model,oymak2010new,oymak2018universality}. Most relevant to us is the work of \citet{thrampoulidis2015regularized}, which introduced the Convex Gaussian Minmax Theorem (CGMT) and developed a framework for the precise analysis of regularized linear regression. Here we apply the GMT/CGMT to study uniform convergence and the norm of the minimal norm interpolator.

\paragraph{Proof sketch of Theorem~\ref{thm:main-generalization}.}

For simplicity, assume here there is no covariance splitting: $\Sigma_2 = \Sigma$. By a change of variable and introducing the Lagrangian, we can rewrite the generalization gap as
\begin{equation} \label{eqn:sketch-1}
    \begin{split}
        \sup_{\substack{w \in \cK \\ Xw = Y}} L(w) &= \sigma^2 + \sup_{\substack{w \in \Sigma^{1/2} (\cK-w^*) \\ Zw = \xi}} \| w\|_2^2\\
        &= \sigma^2 + \sup_{w \in \Sigma^{1/2} (\cK-w^*)} \inf_{\lambda} \, \langle \lambda, Z w - \xi\rangle + \| w\|_2^2
    \end{split}
\end{equation}
where $Z$ is a random matrix with i.i.d. standard normal entries. By GMT\footnote{We ignore a compactness issue here, but this is done rigorously by a truncation argument in \cref{apdx:uniform-convergence:general}.}, we can control the upper tail of the max-min problem above (PO) by the auxiliary problem below (AO), with $G, H \sim N(0, I)$:
\begin{equation}
    \sup_{w \in \Sigma^{1/2} (\cK-w^*)} \inf_{\lambda} \, \| \lambda \|_2 \langle H, w \rangle +  \langle \lambda, G \| w\|_2 - \xi \rangle + \| w\|_2^2 =  \sup_{\substack{w \in \Sigma^{1/2} (\cK-w^*) \\ \norm{ G \norm{w}_2 - \xi}_2 \leq \langle H, w \rangle}}  \norm w_2^2
.\end{equation}
By standard concentration results, we can expect $\| G\|_2^2/n \approx 1$ and $\| \xi \|_2^2/n \approx \sigma^2$, so expanding the second constraint in the AO, we obtain $\| w\|_2^2 + \sigma^2 \leq | \langle H, w \rangle|^2/n$. Plugging into \eqref{eqn:sketch-1}, we have essentially shown that
\begin{equation}
    \sup_{\substack{w \in \cK \\ \hat L(w) = 0}} L(w) \le \sup_{w \in \Sigma^{1/2} (\cK-w^*)} \frac{| \langle H, w \rangle|^2}{n} \le \frac{ \left( \sup_{w \in \Sigma^{1/2} \cK} | \langle H, w \rangle| + |\langle H, \Sigma^{1/2} w^*\rangle| \right)^2}{n}
.\end{equation}
Applying concentration on the right hand side concludes the proof sketch. In situations where the supremum does not sharply concentrate around its mean, we can apply GMT only to the small variance directions of $\Sigma$. This requires a slightly more general version of GMT, which we prove in \cref{apdx:prelim}. We also show the additional terms contributed by the large variance components of $X$ cancel out due to Wishart concentration. This is reflected in the $\beta$ term of our theorem statement.

\paragraph{Proof sketch of Theorem~\ref{thm:interpolator-general}.}

Since the minimal norm problem is \emph{convex-concave}, we can apply the CGMT, which provides a useful direction that GMT cannot. By the same argument as above
\begin{equation}
    \begin{split}
        \inf_{Xw = Y} \| w\| - \| w^* \| &\leq  \inf_{Xw = \xi} \| w\| = \inf_w \sup_{\lambda} \| \Sigma^{-1/2} w \| + \langle \lambda, Zw-\xi \rangle \\
        &\approx \inf_{\| w\|_2^2 + \sigma^2 \leq | \langle H, w \rangle|^2/n} \| \Sigma^{-1/2} w \| = \inf_{\| w\|_{\Sigma}^2 + \sigma^2 \leq | \langle \Sigma^{1/2} H, w \rangle|^2/n} \| w \|
    .\end{split}
\end{equation}
To upper bound the infimum, it suffices to construct a feasible $w$. Consider $w$ of the form $\alpha v$ where $v \in \partial \| \Sigma^{1/2} H \|_*$. Plugging in the constraint, we can choose $\| w \| = \alpha = \sqrt{\sigma^2 \left( \frac{\| \Sigma^{1/2} H \|_*^2}{n} - \| v\|_{\Sigma}^2 \right)^{-1}}$.
Rearranging the terms conclude the proof sketch when there is no covariance splitting. The general proof (in \cref{apdx:norm-bounds:general}) is more technical, but follows the same idea. 

\section{Discussion}
In this work, we prove a generic generalization bound in terms of the Gaussian width and radius of a hypothesis class. We also provide a general high probability upper bound for the norm of the minimal norm interpolator. Combining these results, we recover the sufficient conditions from \citet{bartlett2020benign} in the $\ell_2$ case, confirm the conjecture of \citet{junk-feats} for Gaussian data, and obtain novel consistency results in the $\ell_1$ case. Our results provide concrete evidence that uniform convergence is indeed sufficient to explain interpolation learning, at least in some settings. 

A future direction of our work is to extend the main results to settings with non-Gaussian features; this has been achieved in other applications of the GMT \citep{oymak2018universality}, and indeed we expect that a version of \eqref{eq:spec-ub} likely holds for non-Gaussian data as well.
Another interesting problem is to study uniform convergence of low-norm \emph{near}-interpolators, and characterize the worst-case population error as the norm and training error {both} grow. This could lead to a more precise understanding of early stopping, by connecting the optimization path with the regularization path. Finally, it is unknown whether our sufficient conditions for consistency in \cref{sec:general-norm} are necessary, and it remains a challenge to apply uniform convergence of interpolators to more complex models such as deep neural networks. 

\begin{ack}
Frederic Koehler was supported in part by E. Mossel's Vannevar Bush Faculty Fellowship ONR-N00014-20-1-2826.
Research supported in part by NSF IIS award 1764032, NSF HDR TRIPODS award 1934843, and the Canada CIFAR AI Chairs program.
This work was done as part of the Collaboration on the Theoretical Foundations of Deep Learning (\httpsurl{deepfoundations.ai}).
\end{ack}

\begin{refcontext}[sorting=nyt]
\printbibliography
\end{refcontext}
\clearpage

\appendix

\section*{Appendices}
In the remainder of the paper, we give self-contained proofs of all results from the main text.

In \cref{apdx:prelim}, we introduce some technical results that we will use in our analysis.

In \cref{apdx:uniform-convergence}, we prove the main generalization bound (\cref{thm:main-generalization}) and show its specialization to norm balls (\cref{corr:spec-bound,corr:euclidean-generalization,corr:main-generalization}).

In \cref{apdx:norm-bounds}, we prove upper bounds on the norm of the minimal-norm interpolator for a general norm (\cref{thm:interpolator-general}), and show applications to the Euclidean case (\cref{thm:interpolator-euclidean}).

In \cref{apdx:benign-overfitting}, we show how to combine the previous sets of results to give risk guarantees for the minimal norm interpolators (\cref{corr:min-norm-euclidean,thm:min-norm}).
In particular, \cref{apdx:equivalent} shows the equivalence of conditions for consistency in the Euclidean norm setting.

In \cref{apdx:bp}, we provide full theorem statements and proofs of the results on $\ell_1$ interpolation (basis pursuit) mentioned in \cref{sec:basis-pursuit}.

\section{Preliminaries}\label{apdx:prelim}

We will first give some general results useful to the rest of the proofs. Most are standard, but a few are variations on existing results.

\paragraph{Concentration of Lipschitz functions.} Recall that a function $f : \mathbb{R}^n \to \mathbb{R}$ is $L$-Lipschitz with respect to the norm $\norm\cdot$ if it holds for all $x, y \in \mathbb R^n$ that $|f(x) - f(y)| \le L\|x - y\|$. We use the concentration of Lipschitz functions of a Gaussian.

\begin{theorem}[\cite{van2014probability}, Theorem 3.25]\label{thm:gaussian-concentration}
If $f$ is $L$-Lipschitz with respect to the Euclidean norm and $Z \sim N(0,I_n)$, then
\begin{equation}
    \Pr(|f(Z) - \E f(Z)| \ge t) \le 2e^{-t^2/2L^2}.
\end{equation}

\end{theorem}
We also use a similar result for functions of a uniformly spherical vector \citep[see][Theorem 5.1.4 and Exercise 5.1.12]{vershynin2018high}; we cite a result with sharp constant factor from \citet{ledoux1992heat}.
\begin{theorem}[Spherical concentration; \cite{ledoux1992heat}]\label{thm:sphere-concentration}
If $f$ is $L$-Lipschitz with respect to the Euclidean norm and $Z \sim \text{Uni}(S^{n - 1})$ where $S^{n-1} = \{ u \in \mathbb{R}^n : \|u\| = 1\}$ is the unit sphere, $\text{Uni}(S^{n - 1})$ is the uniform measure on the sphere, and $n \ge 3$, then
\begin{equation}
    \Pr(|f(Z) - \E f(Z)| \ge t) \le 2e^{-(n - 2)t^2/2L^2}.
\end{equation}

\end{theorem}

The following lemma, which we will use multiple timues, says that a $o(n)$-dimensional subspace cannot align with a random spherically symmetric vector. 

\begin{lemma}\label{lem:lowrank-projection}
Suppose that $S$ is a fixed subspace of dimension $d$ in $\mathbb R^n$ with $n \ge 4$, $P_S$ is the orthogonal projection onto $S$,
and $V$ is a spherically symmetric random vector (i.e. $V/\|V\|_2$ is uniform on the sphere). Then
\begin{equation}
    \frac{\|P_S V\|_2}{\|V\|_2} \le \sqrt{d/n} + 2\sqrt{\log(2/\delta)/n}.
\end{equation}

with probability at least $1 - \delta$. Conditional on this inequality holding, we therefore have uniformly for all $s \in S$ that
\begin{equation}  
|\langle s, V \rangle| = |\langle s, P_S V \rangle| \le \|s\|_2 \|P_S V\|_2 \le \|s\|_2\|V\|_2\left(\sqrt{d/n} + 2\sqrt{\log(2/\delta)/n})\right).
\end{equation}
\end{lemma}

\begin{proof}
This is trivial if $d \ge n$, since the left-hand side is at most $1$. Thus assume without loss of generality that $d < n$. By symmetry, it suffices to fix $S$ to be the span of basis vectors $e_1,\ldots,e_d$ and to bound $\|P_S V\|_2$ for $V$ a uniformly random chosen vector from the unit sphere in $\mathbb{R}^n$. Recall that for any coordinate $i$, we have $\E V_i^2 = 1/n$ by symmetry among the coordinates and the fact that $\|V\|_2^2 = 1$ almost surely.
The function $v \mapsto \|P_S v\|_2$ is a $1$-Lipschitz function and $\E\|P_S V\|_2 \le \sqrt{\E \|P_S V\|_2^2} = \sqrt{d/n}$, so by \cref{thm:sphere-concentration} above
\begin{equation*}
    \|P_S V\|_2 \le \sqrt{d/n} + \sqrt{2\log(2/\delta)/(n - 2)})
\end{equation*}
with probability at least $1 - \delta$. Using $n \ge 4$ gives the result.
\end{proof}

The concentration of the Euclidean norm of a Gaussian vector follows from Theorem~\ref{thm:gaussian-concentration}; we state it explicitly below.
\begin{lemma}%
\label{lem:norm-concentration}
Suppose that $Z \sim N(0,I_n)$. Then
\begin{equation}
    \Pr(\left|\|Z\|_2 - \sqrt{n}\right| \ge t) \le 4 e^{-t^2/4}.
\end{equation}
\end{lemma}
\begin{proof}
First we recall the standard fact \citep[see e.g.][]{chandrasekaran2012convex}
that
\begin{equation*}
    \sqrt{n} - 1 \le  \frac{n}{\sqrt{n + 1}} \le \E \|Z\|_2 \le \sqrt{n}.
\end{equation*}
Because the norm is 1-Lipschitz, it follows from Theorem~\ref{thm:gaussian-concentration} that
\begin{equation*}
    \Pr(\left|\|Z\|_2 - \E \|Z\|_2\right| \ge t) \le 2e^{-t^2/2}
\end{equation*}
so
\begin{equation*}
    \Pr(\left|\|Z\|_2 - \sqrt{n}\right| \ge t + 1) \le 2e^{-t^2/2}.
\end{equation*}
Now using that $(t - 1)^2 \ge t^2/2 - 1$ shows
\begin{equation*}
    \Pr(\left|\|Z\|_2 - \sqrt{n}\right| \ge t) \le 2e^{-(t^2/2 - 1)/2} \le 4 e^{-t^2/4}.
\qedhere\end{equation*}
\end{proof}

\paragraph{Wishart concentration.} We recall the notation for the Loewner order on symmetric matrices: $A \preceq B$ means that $B - A$ is positive semidefinite. Let $\sigma_{\min}(A)$ denote the minimum singular value of an arbitrary matrix $A$, and $\sigma_{\max}$ the maximum singular value. Similarly, let $\lambda_{\min}(A)$ denote the minimum eigenvalue. We use $\|A\|_{\OP} = \sigma_{\max}(A)$ to denote the operator norm of matrix $A$.

\begin{theorem}[\cite{vershynin2010rmt}, Corollary~5.35]\label{thm:rmt}
    Let $n,N \in \NN$. Let $A \in \RR^{N \times n}$ be a random matrix with entries i.i.d.\ $N(0,1)$. Then for any $t>0$, it holds with probability at least $1 - 2\exp(-t^2/2)$ that 
    \begin{equation}
        \sqrt{N} - \sqrt{n} - t \leq \sigma_\text{min}(A) \leq \sigma_\text{max}(A) \leq \sqrt{N} + \sqrt{n} + t.
    \end{equation}
\end{theorem}

\begin{corollary}\label{corr:wishart}
Suppose $X_1,\ldots,X_n \sim N(0,\Sigma)$ are independent with $\Sigma : d \times d$ a positive semidefinite matrix, $t > 0$ and $n \ge 4(d + t^2)$. Let $\hat{\Sigma} = \frac{1}{n} \sum_i X_i X_i^T$ be the empirical covariance matrix. Then with probability at least $1 - \delta$, 
\begin{equation}
    (1 - \epsilon) \Sigma \preceq \hat{\Sigma} \preceq (1 + \epsilon) \Sigma
\end{equation}
with $\epsilon = 3\sqrt{d/n} + 3\sqrt{2\log(2/\delta)/n}$.
\end{corollary}

\begin{proof}
Let $X : n \times d$ be the random matrix with rows $X_1,\ldots,X_n$ so that $\hat{\Sigma} = \frac{1}{n} X^T X$. By equality in distribution, we can take $Z : n \times d$ to have $N(0,1)$ independent entries and write $ X = Z \Sigma^{1/2} $ and
\begin{equation*}
    \Sigma^{-1/2} \hat{\Sigma} \Sigma^{-1/2} = \frac{1}{m} \Sigma^{-1/2} X^T X \Sigma^{-1/2} = \frac{1}{n} Z^T Z.
\end{equation*}

By definition of singular values, from \cref{thm:rmt} the eigenvalues of $Z^T Z/n$ are bounded between $(1-\sqrt{d/n}-\sqrt{t^2/n})^2$ and $(1+\sqrt{d/n}+\sqrt{t^2/n})^2$. Since $1-(1-x)^2 \le (1+x)^2 - 1$, using the inequality $(1 + x)^2 \le 1 + 3x$ for $x \in [0,1]$, we have shown that
\begin{equation*}
    \|I - \Sigma^{-1/2} \hat{\Sigma} \Sigma^{-1/2}\|_{\OP} \le (1 + \sqrt{d/n} + \sqrt{t^2/n})^2 - 1 \le 3\sqrt{d/n} + 3\sqrt{t^2/n}.
\end{equation*}
Rewriting and taking $t^2 = 2\log(2/\delta)$ gives the result. 
\end{proof}

\paragraph{Gaussian Minmax Theorem.} The following result is Theorem 3 of \citet{thrampoulidis2015regularized}, known as the Convex Gaussian Minmax Theorem or CGMT (see also Theorem 1 in the same reference). 
As explained there, it is a consequence of the main result of \citet{gordon1985some}, known as Gordon's Theorem or the Gaussian Minmax Theorem. Despite the name, convexity is only required for one of the theorem's conclusions. 

\begin{theorem}[Convex Gaussian Minmax Theorem; \cite{thrampoulidis2015regularized,gordon1985some}]\label{thm:gmt}
Let $Z : n \times d$ be a matrix with i.i.d. $N(0,1)$ entries and suppose $G \sim N(0,I_n)$ and $H \sim N(0,I_d)$ are independent of $Z$ and each other. Let $S_w,S_u$ be compact sets and $\psi : S_w \times S_u \to \mathbb{R}$ be an arbitrary continuous function.
Define the \emph{Primary Optimization (PO)} problem
\begin{equation}
    \Phi(Z) := \min_{w \in S_w} \max_{u \in S_u} \langle u, Z w \rangle + \psi(w,u)
\end{equation}
and the \emph{Auxiliary Optimization (AO)} problem
\begin{equation}
    \phi(G,H) := \min_{w \in S_w} \max_{u \in S_u} \|w\|_2\langle G, u \rangle + \|u\|_2 \langle H, w \rangle + \psi(w,u).
\end{equation}
Under these assumptions, $\Pr(\Phi(Z) < c) \le 2 \Pr(\phi(G,H) \le c)$ for any $c \in \mathbb{R}$.

Furthermore, if we suppose that $S_w,S_u$ are convex sets and $\psi(w,u)$ is convex in $w$ and concave in $u$, then $\Pr(\Phi(Z) > c) \le 2 \Pr(\phi(G,H) \ge c)$. 
\end{theorem}
In other words, the first conclusion says that high probability lower bounds on the auxiliary optimization $\phi(G,H)$ imply high probability lower bounds on the primary optimization $\Phi(Z)$. Importantly, this direction holds without any convexity assumptions. Under the additional convexity assumptions, the second conclusion gives a similar comparison of high probability upper bounds. 

In our analysis, we need a slightly more general statement of the Gaussian Minmax Theorem than \cref{thm:gmt}: we need the minmax formulation to include additional variables which only affect the deterministic term in the minmax problem. It's straightforward to prove this result by repeating the argument in \citet{thrampoulidis2015regularized}; below we give an alternative proof which reduces to \cref{thm:gmt}, by introducing extremely small extra dimensions to contain the extra variables. Intuitively, this works because the statement of the GMT allows for arbitrary continuous functions $\psi$, with no dependence on their quantitative smoothness. 

\begin{theorem}[Variant of GMT] \label{thm:gmt-mod}
Let $Z : n \times d$ be a matrix with i.i.d. $N(0,1)$ entries and suppose $G \sim N(0,I_n)$ and $H \sim N(0,I_d)$ are independent of $Z$ and each other. Let $S_W,S_U$ be compact sets in $\mathbb{R}^d \times \mathbb{R}^{d'}$ and $\mathbb{R}^{n} \times \mathbb{R}^{n'}$ respectively, and let $\psi : S_W \times S_U \to \mathbb{R}$ be an arbitrary continuous function.
Define the \emph{Primary Optimization (PO)} problem 
\begin{equation}
    \Phi(Z) := \min_{(w, w') \in S_W} \max_{(u,u') \in S_U} \langle u, Z w \rangle + \psi( (w,w'),(u,u'))
\end{equation}
and the \emph{Auxiliary Optimization (AO)} problem
\begin{equation}
     \phi(G,H) := \min_{(w, w') \in S_W} \max_{(u,u') \in S_U} \|w\|_2\langle G, u \rangle + \|u\|_2 \langle H, w \rangle + \psi((w,w'),(u,u')).
\end{equation}
Under these assumptions, $\Pr(\Phi(Z) < c) \le 2 \Pr(\phi(G,H) \le c)$ for any $c \in \mathbb{R}$.
\end{theorem}

\begin{proof}
Let $\epsilon \in (0, 1)$ be arbitrary and 
\[ S_{W,\epsilon} := \{ (w, \epsilon w') : (w, w') \in S_W \}, \quad S_{U,\epsilon} := \{ (u, \epsilon u') : (u,u') \in S_U \}. \]
Define $\psi_{\epsilon} ((w, w'),(u, u')) := \psi((w,\frac{1}{\epsilon}w'),(u,\frac{1}{\epsilon} u'))$ so that if $W = (w, \epsilon w')$ and $U = (u, \epsilon u')$, then $\psi_{\epsilon}(W, U) = \psi((w,w'),(u,u'))$. We also define $S_w = \{ w \in \R^d: \exists w' \, s.t. \, (w, w') \in S_W \}$. The other sets $S_{w'}, S_u$ and $S_{u'}$ are defined similarly. It is clear that $S_w, S_{w'}, S_u, S_{u'}$, $S_{W,\epsilon}$ and $S_{U,\epsilon}$ are all still compact in their respective topology, and $\psi_{\epsilon}$ is continuous for every $\epsilon > 0$. 

Let $Z' : (n + n') \times (d + d')$ be a matrix with i.i.d. $N(0,1)$ entries such that the top left $n \times d$ matrix is $Z$. Similarly, we define $G'$ to be a $(n + n')$-dimensional Gaussian vector with independent coordinates such that the first $n$ coordinates are $G$, and $H'$ to be a $(d + d')$-dimensional Gaussian vector with independent coordinates such that the first $d$ coordinates are $H$. Next, consider the augmented PO and AO:
\begin{equation} \label{eqn:gmt-mod-augmented}
    \begin{split}
        \Phi_{\epsilon}(Z') &:= \min_{W \in S_{W,\epsilon}} \max_{U \in S_{U,\epsilon}} \langle U, Z' W \rangle + \psi_{\epsilon}(W,U)\\
        \phi_{\epsilon}(G',H') &:=  \min_{W \in S_{W,\epsilon}} \max_{U \in S_{U,\epsilon}} \|W\|_2 \langle G', U \rangle + \|U\|_2 \langle H', W \rangle + \psi_{\epsilon}(W,U)\\
    \end{split}
\end{equation}
It is clear that for a small value of $\epsilon$, the augmented problem will be close to the original problem. More precisely, for every $(w, w') \in S_W$ and $(u, u') \in S_U$
\begin{equation} \label{eqn:gmt-mod-PO}
    \begin{split}
        | \langle (w, \epsilon &w'), Z' (u, \epsilon u') \rangle - \langle w, Z u \rangle| \\
        &= |\epsilon \langle (0, w'), Z' (u, 0) \rangle + \epsilon \langle (w, 0), Z' (0, u') \rangle + \epsilon^2 \langle (0, w'), Z' (0, u') \rangle|\\        
        &\leq \epsilon (R(S_w) + R(S_{w'})) (R(S_u) + R(S_{u'}))\| Z' \|_{\OP}  = \epsilon A \|Z'\|_{\OP} \\
    \end{split}
\end{equation}
where $A := (R(S_w) + R(S_{w'})) (R(S_u) + R(S_{u'}))$ is deterministic and does not depend on $\epsilon$. Similarly, it is routine to check
\begin{equation*} 
    \begin{split}
        \| w \|_2 \langle G, u \rangle &= \| w \|_2 (\langle G', (u, \epsilon u') \rangle - \epsilon \langle G', (0, u') \rangle)\\
        \| u \|_2 \langle H, w \rangle &= \| u \|_2 (\langle H', (w, \epsilon w') \rangle - \epsilon \langle H', (0, w') \rangle)\\
    \end{split}
\end{equation*}
so by the triangle inequality and Cauchy-Schwarz inequality, we have
\begin{equation} \label{eqn:gmt-mod-AO1}
    \begin{split}
        \big\lvert \| (w, \epsilon &w') \|_2 \langle G', (u, \epsilon u') \rangle - \| w \|_2 \langle G, u \rangle \big\rvert \\
        &\leq \epsilon R(S_{w'}) \| G' \|_2 (R(S_u) + \epsilon R(S_{u'})) + \epsilon R(S_w) \| G' \|_2 R(S_{u'})
        \leq \epsilon A \| G' \|_2 \\
    \end{split}
\end{equation}
and
\begin{equation} \label{eqn:gmt-mod-AO2}
    \begin{split}
        \big\lvert \| (u, \epsilon &u') \|_2 \langle H', (w, \epsilon w') \rangle - \| u \|_2 \langle H, w \rangle \big\rvert \\
        \leq \, &\epsilon R(S_{u'}) \| H' \|_2 (R(S_w) + \epsilon R(S_{w'})) + \epsilon R(S_u) \| H' \|_2 R(S_{w'}) \leq \epsilon A \| H' \|_2 \\
    \end{split}
\end{equation}
From \eqref{eqn:gmt-mod-PO}, it follows that 
\begin{equation}\label{eqn:Phi-apx}
\left|\Phi_{\epsilon}(Z') - \Phi(Z)\right| \le \epsilon A \|Z'\|_{\OP}. \end{equation}
Similarly, from \eqref{eqn:gmt-mod-AO1} and \eqref{eqn:gmt-mod-AO2}, it follows that
\begin{equation}\label{eqn:phi-apx} 
|\phi_{\epsilon}(G',H') - \phi(G,H)| \le \epsilon A(\|G'\|_2 + \|H'\|_2). 
\end{equation}

Approximating the original PO and AO by \eqref{eqn:gmt-mod-augmented} allows us to directly apply the Gaussian Minmax Theorem. For any $c \in \R$, we have
\begin{equation*}
    \begin{split}
        \Pr(\Phi(Z) < c) 
        &\le \Pr(\Phi_{\epsilon}(Z') < c + \sqrt{\epsilon}) + \Pr( \epsilon A \|Z'\|_{\OP} > \sqrt{\epsilon}) \\
        &\le 2\Pr(\phi_{\epsilon}(G',H') \le c + \sqrt{\epsilon}) + \Pr( \epsilon A \|Z'\|_{\OP} > \sqrt{\epsilon})  \\
        &\le 2\Pr(\phi(G',H') \le c + 2\sqrt{\epsilon}) + 2 \Pr \left( \epsilon A(\|G'\|_2 + \|H'\|_2) > \sqrt{\epsilon} \right)  \\
        &\quad + \Pr( \epsilon A \|Z'\|_{\OP} > \sqrt{\epsilon})  \\
        &\le 2\Pr(\phi(G',H') \le c + 2\sqrt{\epsilon}) + 2 \Pr \left( \|G'\|_2> \frac{1}{2 A \sqrt{\epsilon}} \right) \\
        &\quad + 2 \Pr \left( \|H'\|_2> \frac{1}{2 A \sqrt{\epsilon}} \right) + \Pr \left( \|Z'\|_{\OP} > \frac{1}{A \sqrt{\epsilon}} \right)  \\
    \end{split}
\end{equation*}

where we used \eqref{eqn:Phi-apx} in the first inequality, \cref{thm:gmt} in the second inequality, and \eqref{eqn:phi-apx} in the last inequality. This holds for arbitrary $\epsilon > 0$ and taking the limit $\epsilon \to 0$ shows the result, because the CDF is right continuous \citep{durrett2019probability} and the remaining terms go to zero by standard concentration inequalities (\cref{lem:norm-concentration,thm:rmt}).
\end{proof}

\section{Uniform Convergence Bounds}
\label{apdx:uniform-convergence}

We will now prove the main generalization bound, as well as its special cases in norm balls and specifically Euclidean norm balls.

\subsection{General case: Proof of Theorem~\ref{thm:main-generalization}}
\label{apdx:uniform-convergence:general}

For convenience, we restate the definition of covariance splitting here:
\covsplit*

It follows from our definition that $\Sigma_1 \Sigma_2 = 0$. Although our results in \cref{apdx:norm-bounds} requires this orthogonality condition (in particular, \cref{lem:norm-auxiliary}), we note that all of our results here in \cref{apdx:uniform-convergence} continue to hold as long as $\Sigma = \Sigma_1 + \Sigma_2$ and both $\Sigma_1, \Sigma_2$ are positive semi-definite. To apply the Gaussian Minimax Theorem, we first formulate the generalization gap as an optimization problem in terms of a random matrix with $N(0,1)$ entries. 

\begin{lemma}\label{lem:rewrite-gap}
Under the model assumptions in \eqref{eqn:model}, let $\mathcal{K}$ be an arbitrary compact set and $\Sigma = \Sigma_1 \oplus \Sigma_2$. Define the primary optimization problem (PO) as
\begin{equation}
    \Phi := \max_{\substack{(w_1, w_2) \in \cS \\ Z_1 w_1 + Z_2 w_2 = \xi}} \| w_1 \|_2^2 + \| w_2 \|_2^2
\end{equation}
where
\begin{equation}
    \cS = \{ (w_1, w_2): \exists w \in \cK \, s.t. \, w_1 = \Sigma_1^{1/2} (w-w^*) \text{ and } w_2 = \Sigma_2^{1/2} (w-w^*) \}
\end{equation}
and $Z_1, Z_2$ are both $n \times d$ random matrices with i.i.d. standard normal entries independent of $\xi$ and each other. Then the generalization gap of interpolators is equal in distribution to the sum of the Bayes risk and the PO:
\begin{equation}
    \max_{w \in \cK, \hat{L}(w) = 0} L(w) - \hat{L}(w) \overset{\cD}{=} \sigma^2 + \Phi
.\end{equation}
\end{lemma}

\begin{proof}
Recall that $L(w) = \sigma^2 + \|w - w^*\|_{\Sigma}^2 $ and $\hat{L}(w) = 0$ is equivalent to $Y = Xw$. Observe that 
\begin{equation*}
    X \overset{\cD}{=} Z_1 \Sigma_1^{1/2} + Z_2 \Sigma_2^{1/2} \quad \text{and} \quad \| w \|_{\Sigma}^2 = \| w \|_{\Sigma_1}^2 + \| w \|_{\Sigma_2}^2
\end{equation*}
so we can decompose
\begin{equation*}
    \begin{split}
        \max_{w \in \cK, \hat{L}(w) = 0} L(w) - \hat{L}(w) &= \sigma^2 + \max_{w \in \cK, Y = Xw} \|w - w^*\|_{\Sigma}^2\\
        &= \sigma^2 + \max_{w \in \cK, X(w-w^*) = \xi} \|w - w^*\|_{\Sigma}^2\\
        &\overset{\cD}{=} \sigma^2 + \max_{\substack{w \in \cK - w^*\\ (Z_1 \Sigma_1^{1/2} + Z_2 \Sigma_2^{1/2}) w = \xi}} \|w\|_{\Sigma_1}^2 + \|w\|_{\Sigma_2}^2 = \sigma^2 + \Phi
    . \qedhere\end{split}
\end{equation*}
\end{proof}

\begin{lemma}[Application of GMT]\label{lem:gmt-app}
In the same setting as \cref{lem:rewrite-gap}, let $G \sim N(0,I_n), H \sim N(0,I_d)$ be Gaussian vectors independent of $Z_1, Z_2, \xi$ and each other. With the same definition of $\cS$, define the auxiliary optimization problem (AO) as
\begin{equation} \label{eqn:phi}
    \phi := \max_{\substack{(w_1, w_2) \in \cS \\ \norm{\xi - Z_1 w_1 - G \|w_2\|_2 }_2 \leq \langle w_2, H \rangle}} \| w_1 \|_2^2 + \| w_2 \|_2^2
\end{equation}
Then it holds that
\begin{equation}
    \Pr(\Phi > t \, | \, Z_1, \xi) \le 2 \Pr(\phi \ge t \, | \, Z_1, \xi),
\end{equation}
and taking expectations we have
\begin{equation}
    \Pr(\Phi > t ) \le 2 \Pr(\phi \ge t ).
\end{equation}
\end{lemma}

\begin{proof}
By introducing Lagrange multipliers, we have
\begin{equation*}
    \begin{split}
        \Phi &= \max_{(w_1, w_2) \in \cS} \min_{\lambda} \, \| w_1 \|_2^2 + \| w_2 \|_2^2 + \langle \lambda, Z_2w_2 - (\xi-Z_1w_1)\rangle \\
        &= \max_{(w_1, w_2) \in \cS} \min_{\lambda} \, \langle \lambda, Z_2w_2 \rangle + \| w_1 \|_2^2 + \| w_2 \|_2^2 - \langle \lambda, \xi-Z_1w_1\rangle.\\
    \end{split}
\end{equation*}

By independence, the distribution of $Z_2$ remains the same after conditioning on $Z_1$ and $\xi$ and the randomness in $\Phi$ comes solely from $Z_2$. Since the mapping from $w$ to $(w_1, w_2)$ is continuous and $\cK$ is compact, $\cS$ is compact. To apply \cref{thm:gmt-mod}, we can take $\psi(w_1, w_2, \lambda) = \| w_1 \|_2^2 + \| w_2 \|_2^2 - \langle \lambda, \xi-Z_1w_1\rangle$, which is clearly continuous. The only challenge is that the domain of $\lambda$ is not compact, but we can handle it by a truncation argument. Define
\begin{equation}
    \Phi_r := \max_{(w_1, w_2) \in \cS} \min_{\| \lambda \| \leq r} \, \langle \lambda, Z_2w_2 \rangle + \| w_1 \|_2^2 + \| w_2 \|_2^2 - \langle \lambda, \xi-Z_1w_1\rangle
\end{equation}
and observe that $\Phi \leq \Phi_r$, since the minimum in the definition of $\Phi_r$ ranges over a smaller set. The AO associated with $\Phi_r$ is
\begin{equation}
    \begin{split}
        \phi_r :&= \max_{(w_1, w_2) \in \cS} \min_{\| \lambda \| \leq r} \, 
        \| w_2 \|_2 \langle G, \lambda \rangle + \| \lambda \|_2 \langle H, w_2 \rangle + \| w_1 \|_2^2 + \| w_2 \|_2^2 - \langle \lambda, \xi-Z_1w_1\rangle \\
        &= \max_{(w_1, w_2) \in \cS} \min_{\| \lambda \| \leq r} \, 
        \| \lambda \|_2 \langle H, w_2 \rangle - \langle \lambda, \xi-Z_1w_1-G \| w_2 \|_2\rangle  + \| w_1 \|_2^2 + \| w_2 \|_2^2 \\
        &= \max_{(w_1, w_2) \in \cS} \min_{0 \leq \lambda \leq r} \, 
         \lambda \left( \langle H, w_2 \rangle - \| \xi-Z_1w_1-G \| w_2 \|_2 \|_2 \right) + \| w_1 \|_2^2 + \| w_2 \|_2^2 .\\
    \end{split}
\end{equation}

We observe that the untruncated auxiliary problem $\phi$ from \eqref{eqn:phi} has a completely analogous form:
\begin{equation*}
    \phi = \max_{(w_1, w_2) \in \cS} \min_{\lambda \geq 0} \, 
         \lambda \left( \langle H, w_2 \rangle - \| \xi-Z_1w_1-G \| w_2 \|_2 \|_2 \right) + \| w_1 \|_2^2 + \| w_2 \|_2^2. \\
\end{equation*}
This is because if $\langle H, w_2 \rangle - \| \xi-Z_1w_1-G \| w_2 \|_2 \|_2 \geq 0$ then the minimum is achieved at $\lambda = 0$, and if $w_1, w_2$ do not satisfy the constraint then taking $\lambda \to \infty$ sends the minimum to $-\infty$. From this formulation, we see that $\phi \le \phi_r \le \phi_s$ for any $r \ge s \ge 0$ since the minimum is taken over a larger set as $r$ grows, and is unconstrained in $\phi$. 

The proof that $\lim_{r \to \infty} \phi_r = \phi$ is an exercise in real analysis, which splits into two cases: 
\begin{enumerate}
    \item The auxiliary problem $\phi$ is infeasible. In this case, we know that for all $(w_1, w_2) \in \cS$ 
    \begin{equation*}
        \langle H, w_2 \rangle - \| \xi-Z_1w_1-G \| w_2 \|_2 \|_2 < 0.
    \end{equation*}
    By compactness of $\cS$ and continuity of the right hand side, there exists $\mu = \mu(\xi,Z_1,G,H) < 0$ (in particular, independent of $r$) such that
    \begin{equation*}
        \langle H, w_2 \rangle - \| \xi-Z_1w_1-G \| w_2 \|_2 \|_2 \leq \mu.
    \end{equation*}
    Therefore, we show
    \begin{equation*}
        \begin{split}
            \phi_r &\leq \max_{(w_1, w_2) \in \cS} \min_{0 \leq \lambda \leq r} \, \lambda \mu + \| w_1 \|_2^2 + \| w_2 \|_2^2 \\
            &= r\mu + \max_{(w_1, w_2) \in \cS} \| w_1 \|_2^2 + \| w_2 \|_2^2
        .\end{split}
    \end{equation*}
    Since the second term is bounded and has no dependence on $r$, taking $r \to \infty$ we have $\phi_r \to -\infty$ as desired (since $\phi = -\infty$ by definition). 
    
    \item The auxiliary problem $\phi$ is feasible. In this case, we can let $(w_1(r), w_2(r)) \in \cS$ be an arbitrary maximizer achieving the objective $\phi_r$ for each $r \ge 0$ by compactness. By compactness again, the sequence $(w_1(r),w_2(r))_{r = 1}^{\infty}$ at positive integer values of $r$ has a subsequential limit $(w_1(\infty), w_2(\infty)) \in \cS$, i.e. this point satisfies $(w_1(\infty),w_2(\infty)) = \lim_{n \to \infty} (w_1(r_n), w_2(r_n))$ for some sequence $r_n$ satisfying $r_n \ge n$.
    
    Suppose that $(w_1(\infty), w_2(\infty))$ does not satisfy the last constraint defining $\phi$, then by continuity, there exists $\mu < 0$ and a sufficiently small $\epsilon > 0$ such that for all $\|w_1 - w_1(\infty)\|_2 \leq \epsilon$ and $\|w_2 - w_2(\infty)\|_2 \leq \epsilon$, we have
    \begin{equation*}
        \langle H, w_2 \rangle - \| \xi-Z_1w_1-G \| w_2 \|_2 \|_2 \leq \mu.
    \end{equation*}
    This implies that for sufficiently large $n$, we have
    \begin{equation*}
        \langle H, w_2(r_n) \rangle - \| \xi-Z_1w_1(r_n)-G \| w_2(r_n) \|_2 \|_2 \leq \mu
    \end{equation*}
    and
    \begin{equation*}
        \begin{split}
            \phi_{r_n} &\leq r_n \mu + \| w_1(r_n) \|_2^2 + \| w_2(r_n) \|_2^2 \\
            &\leq r_n \mu + \max_{(w_1, w_2) \in \cS} \| w_1 \|_2^2 + \| w_2 \|_2^2 \\
        \end{split}
    \end{equation*}
    so $\phi_{r_n} \to -\infty$ -- but this is impossible, since considering any feasible element of $\phi$ we can show that $\phi_{r_n} \ge 0$. By contradiction, we find that $(w_1(\infty),w_2(\infty))$ is feasible for $\phi$. 
    
    By taking $\lambda = 0$ in the definition of $\phi_r$ we have
    \begin{equation*}
        \phi_{r_n} \leq \|w_1(r_n)\|_2^2 + \|w_2(r_n)\|_2^2
    .\end{equation*}
    By continuity, we show that
    \begin{equation*}
        \begin{split}
            \lim\sup_{n \to \infty} \phi_{r_n} &\leq \lim_{n \to \infty} \|w_1(r_n)\|_2^2 + \|w_2(r_n)\|_2^2 \\
            &= \|w_1(\infty)\|_2^2 + \|w_2(\infty)\|_2^2 \leq \phi. 
        \end{split}
    \end{equation*}
    Since $\phi_{r_n} \geq \phi$, the limit of $\phi_{r_n}$ exists and equals $\phi$. We can conclude that $\lim_{r \to \infty} \phi_r = \phi$ because $\phi_r$ is a monotone decreasing function of $r$.
\end{enumerate}

By our version of the Gaussian Minmax Theorem, \cref{thm:gmt-mod},
\begin{equation*}
    \Pr(\Phi_r > t | Z_1,\xi) = \Pr(-\Phi_r < -t | Z_1,\xi) \le 2 \Pr(-\phi_r \le -t | Z_1,\xi) = 2 \Pr(\phi_r \ge t | Z_1,\xi)
\end{equation*}
We introduce the negative signs here because we have originally a max-min problem instead of a min-max problem. This means the comparison theorem gives an upper bound, instead of a lower bound, on the quantity of interest. 

Finally, we can conclude 
\begin{equation*}
    \Pr(\Phi > t | Z_1, \xi) \le \inf_{r \ge 0} \Pr(\Phi_r > t | Z_1, \xi) \le 2 \inf_{r \ge 0} \Pr(\phi_r \ge t | Z_1, \xi) \le 2 \Pr(\phi \ge t | Z_1,\xi).
\end{equation*}
where the last step uses continuity (from above) of probability measure and the fact that $\phi_r$ monotonically decreases to $\phi$ almost surely.
\end{proof}

Recall the definition of Gaussian width and radius:

\gwidthrad*

It remains to analyze the auxiliary problem, which we do in the following lemma:
\begin{lemma}\label{lem:aux-analysis}
Let $\beta = 33 \sqrt{\frac{\log(32/\delta)}{n}} + 18 \sqrt{\frac{\rank(\Sigma_1) }{n}}$. If $n$ is sufficiently large such that $\beta \leq 1$, then with probability at least $1 - \delta$, it holds that
\begin{equation} \label{eqn:aux-bound}
    \phi \leq \frac{1+\beta}{n} (W(\Sigma_2^{1/2} \cK) + \rad(\Sigma_2^{1/2} \cK)\sqrt{2\log(16/\delta)} + \|w^*\|_{\Sigma_2} \sqrt{2\log(16/\delta)})^2 - \sigma^2
.\end{equation}
\end{lemma}

\begin{proof}
For notational simplicity, define
\begin{equation*}
    \begin{split}
        \alpha &:= 2\sqrt{\frac{\log(32/\delta)}{n}} \\
        \gamma &:= 3 \sqrt{\frac{\rank(\Sigma_1) }{n}} + 3\sqrt{\frac{2\log(16/\delta)}{n}} \\
        \rho &:= \sqrt{\frac{\rank(\Sigma_1) + 1}{n}} + 2\sqrt{\frac{\log(16/\delta)}{n}}
    .\end{split}
\end{equation*}

By a union bound, the following collection of events, which together we call $\cE$, occurs with probability at least $1 - \delta$: 
\begin{enumerate}
    \item (Approximate orthogonality.) By \cref{lem:lowrank-projection}, uniformly over all $w_1 \in \Sigma_1^{1/2} (\cK - w^*)$ and $a \in \mathbb{R}$, it holds that 
    \begin{equation} \label{eqn:G-orthogonal} 
        |\langle \xi a - Z_1 w_1, G \rangle| \leq \|\xi a - Z_1 w_1\|_2 \|G\|_2 \rho
    \end{equation}
    and
    \begin{equation}\label{eqn:xi-orthogonal}
    |\langle \xi, Z_1 w_1 \rangle| \leq \|\xi\|_2 \|Z_1 w_1\|_2 \rho. 
    \end{equation}

    \item (Approximate isometry.) By \cref{corr:wishart}, uniformly over all $w_1 \in \Sigma_1^{1/2} (\cK - w^*)$, it holds that
    \begin{equation}\label{eqn:isometry}
    (1 - \gamma) \| w_1 \|_2^2 \leq \frac{\| Z_1 w_1 \|_2^2}{n}\leq (1 + \gamma) \| w_1 \|_2^2.
    \end{equation}

    \item (Typical norm of $G$ and $\xi$.) By \cref{lem:norm-concentration}, it holds that
    \begin{equation}\label{eqn:g-norm}
         - \alpha \le \frac{1}{\sqrt{n}}\|G\|_2 - 1 \le \alpha
    \end{equation}
    and
    \begin{equation}\label{eqn:xi-norm}
         - \alpha \sigma \le \frac{1}{\sqrt{n}}\|\xi\|_2 - \sigma \le \alpha \sigma.
    \end{equation}
    
    \item (Typical size of $\langle \Sigma_2^{1/2} w^*, H \rangle$.) By the standard Gaussian tail bound $\Pr (|Z| \geq t) \leq 2e^{-t^2/2}$, it holds that
    \begin{equation}\label{eqn:signal-size} 
    |\langle \Sigma_2^{1/2} w^*, H \rangle| \le \|w^*\|_{\Sigma_2} \sqrt{2 \log(16/\delta)}
    \end{equation}
    because the marginal law of $\langle \Sigma_2^{1/2} w^*, H \rangle$ is $N(0, \|w^*\|_{\Sigma_2}^2)$.
    
    \item (Gaussian process concentration.) 
    By \cref{thm:gaussian-concentration}, it holds that
    \begin{equation} \label{eqn:gp-concentration} 
    \max_{w_2 \in \Sigma_2^{1/2} \cK} |\langle w_2, H \rangle| \leq W(\Sigma_2^{1/2} \cK) + \rad(\Sigma_2^{1/2} \cK) \sqrt{2\log(16/\delta)}
    \end{equation}
    because $ \max_{w_2 \in \Sigma_2^{1/2} \cK} |\langle u_2, H \rangle| $ is a $\rad(\Sigma_2^{1/2} \cK)$-Lipschitz function of $H$.
\end{enumerate}

From now on, the argument is conditional on the event $\mathcal{E}$ defined above. By squaring the last constraint in the definition of $\phi$ we see that
\begin{equation*}
    \begin{split}
        \langle  w_2, H \rangle^2 &\geq \|\xi - Z_1 w_1 - \|w_2\|_2 G\|_2^2 \\
        &= \norm{\xi - Z_1 w_1}_2^2 + \|w_2\|_2^2 \norm{G}_2^2 - 2\langle \xi - Z_1 w_1, \|w_2\|_2 G \rangle\\
        &\geq (1-\rho) [\norm{\xi - Z_1 w_1}_2^2 + \|w_2\|_2^2 \norm{G}_2^2]\\
    \end{split}
\end{equation*}
where in the last line we used \eqref{eqn:G-orthogonal} and the AM-GM inequality ($ab \le a^2/2 + b^2/2$). Rearranging gives the inequality
\begin{equation*}
    \begin{split}
        \|w_2\|_2^2 &\leq \frac{(1-\rho)^{-1}\langle  w_2, H \rangle^2 - \norm{\xi - Z_1 w_1}_2^2 }{\|G\|_2^2}\\
        &\leq \frac{(1-\rho)^{-1}\langle  w_2, H \rangle^2 - (1-\rho) [\norm{\xi}_2^2 + \norm{Z_1 w_1}_2^2] }{\|G\|_2^2}\\
        &\leq \frac{(1-\rho)^{-1}\langle  w_2, H \rangle^2 - (1-\rho) [\norm{\xi}_2^2 + \norm{Z_1 w_1}_2^2] }{(1 - \alpha)^2 n}\\
        &\leq - \frac{(1 - \gamma)(1 - \rho)}{(1 - \alpha)^2} \norm{w_1}_2^2 + \frac{(1-\rho)^{-1}\langle  w_2, H \rangle^2 - (1-\rho) \norm{\xi}_2^2  }{(1 - \alpha)^2 n}\\
    \end{split}
\end{equation*}
where in the second inequality we used \eqref{eqn:xi-orthogonal} and the AM-GM inequality again, in the third inequality we used \eqref{eqn:g-norm} and in the last inequality we used \eqref{eqn:isometry}. This shows
\begin{equation*}
    \begin{split}
        (1 - \gamma)(1 - \rho)(\norm{w_1}_2^2+\norm{w_2}_2^2) &\leq \frac{(1 - \gamma)(1 - \rho)}{(1 - \alpha)^2} \norm{w_1}_2^2 + \|w_2\|_2^2 \\
        &\leq \frac{(1-\rho)^{-1}\langle  w_2, H \rangle^2 - (1-\rho) \norm{\xi}_2^2  }{(1 - \alpha)^2 n} .\\
    \end{split}
\end{equation*}

Dividing through by the first two factors on the left hand side and plugging in \eqref{eqn:xi-norm} gives
\begin{equation*}
    \begin{split}
        \norm{w_1}_2^2+\norm{w_2}_2^2 &\leq \frac{(1-\rho)^{-2}\langle  w_2, H \rangle^2 - \norm{\xi}_2^2  }{(1 - \gamma)(1 - \alpha)^2 n} \\
        &\leq \frac{1}{(1 - \gamma)(1 - \alpha)^2(1-\rho)^2} \frac{\langle  w_2, H \rangle^2}{n} - \frac{\sigma^2}{1-\gamma}.\\
    \end{split}
\end{equation*}

We can simplify the first term by defining $\beta = (1 - \gamma)^{-1}(1 - \alpha)^{-2}(1-\rho)^{-2} - 1$ and the second term by observing $-\frac{\sigma^2}{1-\gamma} \leq - \sigma^2$. Finally, plugging into \eqref{eqn:phi} gives
\begin{equation*}
    \begin{split}
        \phi &\leq \max_{(w_1, w_2) \in S} (1+\beta) \frac{\langle w_2, H \rangle^2}{n} - \sigma^2 \\
        &= \frac{1+\beta}{n} \max_{w_2 \in \Sigma_2^{1/2}(\cK-w^*)} |\langle w_2, H \rangle|^2 - \sigma^2\\
        &\leq \frac{1+\beta}{n} \left( \max_{w_2 \in \Sigma_2^{1/2}\cK} |\langle w_2, H \rangle| + | \langle \Sigma_2^{1/2}w^*, H \rangle| \right)^2 - \sigma^2\\
    \end{split}
\end{equation*}
by the triangle inequality, and \eqref{eqn:aux-bound} follows by \eqref{eqn:signal-size} and \eqref{eqn:gp-concentration}. To deduce the explicit bound for $\beta$, first use that
\begin{equation*}
    (1-\alpha)^2 = 1 - 2\alpha + \alpha^2 \geq 1-2\alpha
\end{equation*}
and similarly $(1-\rho)^2 \geq 1-2\rho$ to show
\begin{equation*}
    \frac{1}{(1 - \gamma)(1 - \alpha)^2(1 - \rho)^2} \leq \frac{1}{(1 - \gamma)(1 - 2\rho)(1 - 2\alpha)}.
\end{equation*}

If $\gamma, \rho < 1/2$, then
\begin{equation*}
    \begin{split}
        (1 - \gamma)(1 - 2\alpha)(1 - 2\rho) &= 1 - \gamma - 2\alpha - 2\rho +2\gamma\alpha + 2\gamma\rho  + 4\alpha \rho - 4\gamma \alpha \rho \\
        &\geq 1 - \gamma - 2\alpha - 2\rho - 4\gamma \alpha \rho > 1 - 2\gamma - 2\alpha - 2\rho.\\
    \end{split}
\end{equation*}

Provided that $2\gamma + 2\alpha + 2\rho < 1/2$ (which implies that $\gamma, \rho < 1/2$), we can use the inequality $(1-x)^{-1} \leq 1 + 2x$ for $x \in [0, 1/2]$ to show that
\begin{equation*}
    \frac{1}{(1 - \gamma)(1 - \alpha)^2(1 - \rho)^2} \leq \frac{1}{1 - 2\gamma - 2\alpha - 2\rho} \leq  1 + 4\gamma + 4\alpha + 4\rho
\end{equation*}
and thus we can choose 
\begin{equation*}
    \beta = 33 \sqrt{\frac{\log(32/\delta)}{n}} + 18 \sqrt{\frac{\rank(\Sigma_1) }{n}} \geq 4\gamma + 4\alpha + 4\rho
\qedhere \end{equation*}
\end{proof}

We are finally ready to prove our main generalization bound:

\maingenbound*

\begin{proof}
By \cref{lem:rewrite-gap,lem:gmt-app}, we show that for any $t$
\begin{equation*}
    \Pr \left( \max_{w \in \cK, \hat{L}(w) = 0} L(w) - \hat{L}(w) > t \right) = \Pr (\Phi > t-\sigma^2) \leq 2 \Pr (\phi \geq t-\sigma^2).
\end{equation*}
By \cref{lem:aux-analysis}, the above is upper bounded by $\delta$ if we set $t-\sigma^2$ according to \eqref{eqn:aux-bound} with $\delta$ replaced by $\delta/2$. Observe that the $\sigma^2$ term cancels, and the proof is complete.
\end{proof}

\begin{remark}[Translation-invariant version]
Our generalization guarantee is stated in terms of $W(\cdot)$ and $\rad(\cdot)$, which are not translation-invariant. However, the generalization guarantee of \cref{thm:main-generalization} can be made translation invariant, e.g.\ replacing $W(\Sigma_2^{1/2} \cK)$ by $W(\Sigma_2^{1/2} (\cK - a))$ for an arbitrary $a \in \mathbb{R}^d$, by recentering the problem before applying \cref{thm:main-generalization}, i.e. by subtracting $X a$ from both sides of the
interpolation constraint $X w = X w^* + \xi$.%
\end{remark}
We also note that in \cref{thm:main-generalization}, there is no requirement that $w^* \in \cK$, so the true function may not necessarily lie in the class even if there is no noise ($\sigma = 0$).

\subsection{Specialization to General Norm Balls}

For convenience, we restate the general definition of effective rank.

\geneffrank*

Applying \cref{thm:main-generalization} to an arbitrary norm ball yield the following: 

\maingen*

\begin{proof}
Let $\cK = \{w : \|w\| \le B\}$ in \cref{thm:main-generalization}. It is easy to see that
\begin{equation*}
    W(\Sigma_2^{1/2} \cK) = \E \sup_{\|w\| \le B} |\langle \Sigma_2^{1/2} w, H \rangle| = \E \sup_{\|w\| \le B} |\langle  w, \Sigma_2^{1/2} H \rangle| = B \E \|\Sigma_2^{1/2} H\|_* 
\end{equation*}
and
\begin{equation*}
    R(\Sigma_2^{1/2} \cK) = \sup_{\|w\| \le B} \|\Sigma_2^{1/2} w\|_2 = B \sup_{\|w\| \le 1} \|w\|_{\Sigma_2}.
\end{equation*}

From our definition, it is clear that 
\begin{equation*}
    r_{\norm\cdot}(\Sigma) = \left( \frac{W(\Sigma^{1/2} \cK)}{R(\Sigma^{1/2} \cK)} \right)^2.
\end{equation*}
Observe that
\begin{equation*}
    \|w^*\|_{\Sigma_2} \leq \| w^* \| \sup_{\|w\| \le 1} \|w\|_{\Sigma_2} \leq B \sup_{\|w\| \le 1} \|w\|_{\Sigma_2} = R(\Sigma_2^{1/2} \cK).
\end{equation*}

The two above equations imply 
\begin{equation*}
    \begin{split}
        & W(\Sigma_2^{1/2} \cK) + \rad(\Sigma_2^{1/2} \cK)\sqrt{2\log(32/\delta)} + \|w^*\|_{\Sigma_2} \sqrt{2\log(32/\delta)} \\
        \leq \, &W(\Sigma_2^{1/2} \cK) +2 \sqrt{2\log(32/\delta)} \rad(\Sigma_2^{1/2} \cK) \\
        = \, &W(\Sigma_2^{1/2} \cK) + 2 \sqrt{\frac{2\log(32/\delta)}{r_{\norm\cdot}(\Sigma_2)}} W(\Sigma_2^{1/2} \cK) \\
        = \, & \left( 1 + 2 \sqrt{\frac{2\log(32/\delta)}{r_{\norm\cdot}(\Sigma_2)}} \right) \left( B \E \|\Sigma_2^{1/2} H\|_* \right). \\
    \end{split}
\end{equation*}
Under our assumptions that $\gamma \leq 1$ and $\delta \leq 1/4$, using the inequality $(1+x)(1+y) \leq 1 + x + 2y$ for $x \leq 1$, it is routine to check that
\begin{equation*}
    (1+\beta) \left( 1 + 2 \sqrt{\frac{2\log(32/\delta)}{r_{\norm\cdot}(\Sigma_2)}} \right)^2 \leq 1 + \gamma
.\end{equation*}
Plugging into \cref{thm:main-generalization} concludes the proof.
\end{proof}

\subsection{Special Case: Euclidean Norm}

In the Euclidean setting, the effective ranks are defined as follows:

\euclideffrank*

Due to the small difference between $r(\Sigma)$ and $r_{\norm{\cdot}_2 }(\Sigma)$, our generalization bound below requires a slightly different proof (see discussion in \cref{sec:general-norm}), but the proof strategies are exactly the same.

\euclidgen*

\begin{proof}
The proof is identical to \cref{corr:main-generalization} except for the inconsequential difference between $\E \|\Sigma_2^{1/2} g\|_2 $ and $\Tr(\Sigma_2)^{1/2}$. It is easy to see that 
\begin{equation*}
    W(\Sigma_2^{1/2} \cK) \leq B \Tr(\Sigma_2)^{1/2} \quad \text{and} \quad R(\Sigma_2^{1/2} \cK) = B \| \Sigma_2 \|_{\OP}^{1/2}.
\end{equation*}

By the same argument, we can show that $\|w^*\|_{\Sigma_2} \leq R(\Sigma_2^{1/2} \cK)$ and
\begin{equation*}
    \begin{split}
        & W(\Sigma_2^{1/2} \cK) + \rad(\Sigma_2^{1/2} \cK)\sqrt{2\log(32/\delta)} + \|w^*\|_{\Sigma_2} \sqrt{2\log(32/\delta)} \\
        \leq \, &W(\Sigma_2^{1/2} \cK) +2 \sqrt{2\log(32/\delta)} \rad(\Sigma_2^{1/2} \cK) \\
        \leq \, & B \Tr(\Sigma_2)^{1/2} +2 \sqrt{2\log(32/\delta)} B \| \Sigma_2 \|_{\OP}^{1/2} \\
        = \, & \left( 1 + 2 \sqrt{\frac{2\log(32/\delta)}{r(\Sigma_2)}} \right) B \Tr(\Sigma_2)^{1/2} . \\
    \end{split}
\end{equation*}
Plugging into \cref{thm:main-generalization} concludes the proof.
\end{proof}

Next, by choosing a particular covariance split, we prove the speculative bound from \citet{junk-feats} when the features are Gaussian:

\specbound*

\begin{proof}
By \cref{thm:main-generalization} and the same argument in proof of \cref{corr:euclidean-generalization}, we obtain
\begin{equation*}
    \begin{split}
        \sup_{w \in \cK, Y = X w} L(w)
        &\leq \frac{1 + \beta}{n}\left(B \Tr(\Sigma_2)^{1/2} + B \| \Sigma_2 \|_{\OP}^{1/2} \cdot 2\sqrt{2\log \left(\frac{32}{\delta}\right)} \right)^2 \\
        &\leq \frac{B^2}{n} (1 + \beta) \left( \Tr(\Sigma)^{1/2} + \| \Sigma_2 \|_{\OP}^{1/2} \cdot 6\sqrt{\log (1/\delta)} \right)^2.\\
    \end{split}
\end{equation*}

Let $\Sigma_1$ contain the largest eigenvalues, then we have
\begin{equation*}
    \rank (\Sigma_1) \| \Sigma_2 \|_{\OP} \leq \Tr(\Sigma).
\end{equation*}

Plugging in the inequality shows
\begin{equation*}
    \begin{split}
        \sup_{w \in \cK, Y = X w} L(w) &\leq \frac{B^2 \Tr(\Sigma)}{n} (1 + \beta) \left( 1 +  6\sqrt{\frac{\log (1/\delta)}{\rank (\Sigma_1)}} \right)^2.\\
    \end{split}
\end{equation*}

Therefore, we can pick $\gamma = (1 + \beta) \left( 1 +  6\sqrt{\frac{\log (1/\delta)}{\rank (\Sigma_1)}} \right)^2 - 1$ and it is clear that
\begin{equation*}
    \gamma \lesssim \sqrt{\frac{\log(1/\delta)}{n}} + \sqrt{\frac{\rank(\Sigma_1) }{n}} +  \sqrt{\frac{\log (1/\delta)}{\rank (\Sigma_1)}} 
\end{equation*}
for sufficiently large $n$ and $\rank (\Sigma_1)$. To balance the last two terms, we can pick a covariance split such that $\rank(\Sigma_1)$ is of order $[n \log(1/\delta)]^{1/2}$, which proves the $\sqrt[4]{\log(1/\delta)/n}$ rate. 
\end{proof}

\section{Bounds on the Norm of the Minimal-Norm Interpolator}\label{apdx:norm-bounds}

In this section, we will give bounds -- again based on the Gaussian Minimax Theorem -- for the norm of the minimal norm interpolator, first in general and then in the Euclidean case.

\subsection{General Norms: Proof of Theorem \ref{thm:interpolator-general}} \label{apdx:norm-bounds:general}

Similar to the analysis in the previous section, we first formulate the minimal norm as an optimization problem in terms of a random matrix with $N(0,1)$ entries. Next, we apply the Convex Gaussian Minimax Theorem.

\begin{lemma} \label{lem:rewrite-norm}
Under the model assumptions in \eqref{eqn:model}, let $\Norm\cdot$ be an arbitrary norm and $Z: n \times d$ be a matrix with i.i.d. $N(0,1)$ entries independent of $\xi$. Define the primary optimization problem (PO) as
\begin{equation}
    \Phi := \min_{Zw = \xi} \| \Sigma^{-1/2} w \|
.\end{equation}
Then for any $t$, it holds that
\begin{equation}
    \Pr \left( \, \min_{Xw = Y} \| w \| > t \, \right) \leq \Pr \left( \, \| w^*\| + \Phi > t \, \right)
.\end{equation}
\end{lemma}

\begin{proof}
By equality in distribution, we can write $X = Z \Sigma^{1/2}$. By the triangle inequality and two changes of variables, we have
\begin{equation*}
    \begin{split}
        \min_{Xw = Y} \| w \| &= \min_{Xw = \xi} \| w + w^*\|\\
        &\leq \| w^* \| + \min_{Z \Sigma^{1/2} w = \xi} \| w \| \\
        &= \| w^* \| +  \min_{Zw = \xi} \| \Sigma^{-1/2} w \|
    .\qedhere
    \end{split}
\end{equation*}
\end{proof}

\begin{lemma}[Application of CGMT]\label{lem:gmt-app-norm}
In the same setting as \cref{lem:rewrite-norm}, let $G \sim N(0,I_n), H \sim N(0,I_d)$ be Gaussian vectors independent of $\xi$ and each other. Define the auxiliary optimization problem (AO) as
\begin{equation} \label{eqn:phi-norm}
    \phi := \min_{\| \xi-  \|w\|_2 G \|_2 \leq \langle H, w \rangle} \| \Sigma^{-1/2} w \|
.\end{equation}
Then it holds that
\begin{equation}
    \Pr(\Phi > t \, | \, \xi) \le 2 \Pr(\phi \ge t \, | \, \xi),
\end{equation}
and taking expectations we have
\begin{equation}
    \Pr(\Phi > t ) \le 2 \Pr(\phi \ge t ).
\end{equation}
\end{lemma}

\begin{proof}
By introducing Lagrange multipliers, we have
\begin{equation*}
    \begin{split}
        \Phi &= \min_w \max_{\lambda} \, \| \Sigma^{-1/2} w \| + \langle \lambda, Zw-\xi \rangle \\
        &= \min_w \max_{\lambda} \, \langle \lambda, Zw \rangle + \| \Sigma^{-1/2} w \| - \langle \lambda, \xi \rangle
    .\end{split}
\end{equation*}
By independence, the distribution of $Z$ remains the same after conditioning on $\xi$ and the randomness in $\Phi$ comes solely from $Z$. Therefore, we can apply CGMT in \cref{thm:gmt} with $\psi (w, \lambda) = \| \Sigma^{-1/2} w \| - \langle \lambda, \xi \rangle$ because $\psi$ is \emph{convex-concave}, but we again have the technical difficulty that the domains of $w$ and $\lambda$ are not compact. To overcome this, we will use a double truncation argument. For any $r, t > 0$, we define 
\begin{equation}
    \Phi_r(t) := \min_{\| \Sigma^{-1/2} w \| \leq 2t} \max_{\|\lambda\|_2 \leq r} \, \langle \lambda, Zw \rangle + \| \Sigma^{-1/2} w \| - \langle \lambda, \xi \rangle 
\end{equation}
and the corresponding AO
\begin{equation}
    \begin{split}
        \phi_r(t) :&= \min_{\| \Sigma^{-1/2} w \| \leq 2t} \max_{\| \lambda \|_2 \leq r} \, \| w \|_2 \langle G, \lambda \rangle + \| \lambda \|_2 \langle H, w \rangle + \| \Sigma^{-1/2} w \| - \langle \lambda, \xi \rangle \\
        &= \min_{\| \Sigma^{-1/2} w \| \leq 2t} \max_{\| \lambda \|_2 \leq r} \,  \| \lambda \|_2 \langle H, w \rangle - \langle \lambda, \xi - \| w \|_2 G \rangle + \| \Sigma^{-1/2} w \|  \\
        &= \min_{\| \Sigma^{-1/2} w \| \leq 2t} \max_{0 \leq \lambda \leq r} \,  \lambda \left( \langle H, w \rangle + \| \xi - \| w \|_2 G \|_2 \right) + \| \Sigma^{-1/2} w \|
    .\end{split}
\end{equation}
Note that the optimization in $\Phi_r(t)$ and $\phi_r(t)$ now ranges over compact sets. We will also use an intermediate problem between $\Phi$ and $\Phi_r(t)$, defined as
\begin{equation}
    \begin{split}
        \Phi(t) :&= \min_{\| \Sigma^{-1/2} w \| \leq 2t} \max_{\lambda} \, \langle \lambda, Zw \rangle + \| \Sigma^{-1/2} w \| - \langle \lambda, \xi \rangle \\
        &= \min_{\substack{Zw = \xi \\ \| \Sigma^{-1/2} w \| \leq 2t}} \| \Sigma^{-1/2} w \|
    .\end{split}
\end{equation}
We similarly define the intermediate AO as
\begin{equation}
    \begin{split}
        \phi(t) :&= \min_{\| \Sigma^{-1/2} w \| \leq 2t} \max_{\lambda \geq 0} \,  \lambda \left( \langle H, w \rangle + \| \xi - \| w \|_2 G \|_2 \right) + \| \Sigma^{-1/2} w \|  \\
        &= \min_{\substack{\| \xi-  \|w\|_2 G \|_2 \leq \langle -H, w \rangle \\ \| \Sigma^{-1/2} w \| \leq 2t}} \| \Sigma^{-1/2} w \|
    .\end{split}
\end{equation}
Compared to the definition of $\phi$, we have $\langle -H, w \rangle$ instead of $\langle H, w \rangle$, but this difference is negligible because $H$ is Gaussian. It can be easily seen that the event $\Phi > t$ is the same as $\Phi(t) > t$, and the same holds for $\phi$ and $\phi(t)$. It is also clear that $\phi(t) \geq \phi_r(t)$ and we can connect $\phi_r(t)$ with $\Phi_r(t)$ by CGMT. It remains to show that $\Phi_r(t) \to \Phi(t)$ as $r \to \infty$.

By definition, $\Phi_r(t) \leq \Phi_s(t)$ for $r \leq s$. We consider two cases:
\begin{enumerate}
    \item $\Phi(t) = \infty$, i.e. the minimization problem defining $\Phi(t)$ is infeasible. In this case, we know that for all $\|\Sigma^{-1/2} w\| \le 2t$
    \begin{equation*}
        \| Z w - \xi\|_2 > 0.
    \end{equation*}
    By compactness, there exists $\mu = \mu(Z, \xi) > 0$ (in particular, independent of $r$) such that
    \begin{equation*}
        \| Z w - \xi \|_2 \geq \mu.
    \end{equation*}
    Therefore, considering $\lambda$ along the direction of $ Z w - \xi$ shows that
    \begin{equation*}
        \Phi_r(t) = \min_{\| \Sigma^{-1/2} w \| \leq 2t} \max_{\|\lambda\|_2 \leq r} \, \langle \lambda, Zw - \xi \rangle + \| \Sigma^{-1/2} w \| \geq r\mu 
    \end{equation*}
    so $\Phi_r(t) \to \infty$ as $r \to \infty$. 
    
    \item Otherwise $\Phi(t) < \infty$, i.e.\ the minimization problem defining $\Phi(t)$ is feasible. In this case, we can let $w(r)$ be an arbitrary minimizer achieving the objective $\Phi_r(t)$ for each $r \geq 0$ by compactness. By compactness again, the sequence $\{ w(r) \}_{r=1}^{\infty}$ at positive integer values of $r$ has a subsequential limit $w(\infty)$ such that $\| \Sigma^{-1/2} w(\infty) \| \leq 2t$. Equivalently, there exists an increasing sequence $r_n$ such that $\lim_{n \to \infty} w(r_n) = w(\infty)$.
    
    Suppose for the sake of contradiction that $Z w(\infty) \ne \xi$, then by continuity, there exists $\mu > 0$ and a sufficiently small $\epsilon > 0$ such that for all $\|w - w(\infty) \|_2 \leq \epsilon$
    \begin{equation*}
        \|Zw - \xi \|_2 \geq \mu
    .\end{equation*}
    This implies that for sufficiently large $n$, we have
    \begin{equation*}
        \|Zw(r_n) - \xi \|_2 \geq \mu
    \end{equation*}
    and by the same argument as in the previous case
    \begin{equation*}
        \Phi_{r_n}(t) = \max_{\|\lambda\|_2 \leq r} \, \langle \lambda, Zw(r_n) - \xi \rangle + \| \Sigma^{-1/2} w(r_n) \| \geq r\mu
    \end{equation*}
    so $\Phi_{r_n} \to \infty$, but this is impossible since $\Phi_r(t) \leq \Phi(t) < \infty$. By contradiction, it must be the case that $Zw(\infty) = \xi$. By taking $\lambda = 0$ in the definition of $\Phi_r(t)$, we have
    \begin{equation*}
        \Phi_{r_n}(t) \geq \| \Sigma^{-1/2} w(r_n) \|
    .\end{equation*}
    By continuity, we show that 
    \begin{equation*}
        \lim \inf_{n \to \infty} \Phi_{r_n}(t) \geq \lim_{n \to \infty} \| \Sigma^{-1/2} w(r_n) \| = \| \Sigma^{-1/2} w(\infty)  \| \geq \Phi(t)
    .\end{equation*}
    Since $\Phi_{r_n}(t) \leq \Phi(t)$, the limit of $\Phi_{r_n}(t)$ exists and equals $\Phi(t)$. We can conclude that $\lim_{r \to \infty} \Phi_r(t) = \Phi(t)$ because $\Phi_r(t)$ is an increasing function of $r$.
\end{enumerate}

By the last part of \cref{thm:gmt} (the CGMT),
\begin{equation*}
    \Pr (\Phi_r(t) > t \, | \, \xi) \leq 2 \Pr (\phi_r(t) \geq t \, | \, \xi)
.\end{equation*}
By continuity (from below) of the probability measure, and the fact that $\Phi_r(t)$ monotonically increases to $\Phi(t)$ almost surely, we can conclude
\begin{equation*}
    \begin{split}
        \Pr(\Phi > t \, | \, \xi) &= \Pr(\Phi(t) > t \, | \, \xi) \leq \Pr \left( \cup_r \cap_{r' \geq r} \Phi_{r'}(t) > t  \, | \, \xi \right)\\
        &= \lim_{r \to \infty} \Pr \left( \cap_{r' \geq r} \Phi_{r'}(t) > t \, | \, \xi \right) = \lim_{r \to \infty} \Pr \left(  \Phi_{r}(t) > t \, | \, \xi \right) \\
        &\leq 2 \lim_{r \to \infty} \Pr (\phi_r(t) \geq t \, | \, \xi ) \leq 2 \Pr (\phi(t) \geq t \, | \, \xi)\\
        &= 2 \Pr(\phi \geq t \, | \, \xi)
    .\qedhere \end{split}
\end{equation*}
\end{proof}

It remains to analyze the auxiliary problem, which we do in the following lemma:
\begin{lemma} \label{lem:norm-auxiliary}
For any covariance splitting $\Sigma = \Sigma_1 \oplus \Sigma_2$, denote $P$ as the orthogonal projection matrix onto the space spanned by $\Sigma_2$, and let $v^* = \argmin_{v \in \partial \|\Sigma^{1/2}_2 H\|_*} \norm{v}_{\Sigma_2}$. Assume that there exists $\epsilon_1, \epsilon_2 \geq 0$ such that with probability at least $1-\delta/2$,
\begin{equation} \label{eqn:norm-auxiliary1}
    \norm{v^*}_{\Sigma_2} \leq (1+\epsilon_1) \E \norm{v^*}_{\Sigma_2}
\end{equation}
and 
\begin{equation} \label{eqn:norm-auxiliary2}
    \norm{P v^*}^2 \leq 1 + \epsilon_2.
\end{equation}
Let
\[
    \epsilon = 8 n^{-1/2} + 28 \sqrt{\frac{\log(32/\delta)}{n}} + 8\sqrt{\frac{\log(8/\delta)}{r_{\| \cdot \|}(\Sigma)}} + 2(1+\epsilon_1)^2 \frac{n}{R_{\| \cdot \|}(\Sigma_2)} + 2\epsilon_2
.\]
If $n$ and the effective ranks are sufficiently large such that $\epsilon \leq 1$, then with probability at least $1-\delta$, it holds that
\begin{equation} \label{eqn:aux-bound2}
    \phi^2 \leq (1+\epsilon) \, \sigma^2 \frac{n}{(\E \| \Sigma^{1/2}_2 H \|_*)^2} 
\end{equation}
\end{lemma}

\begin{proof}
For notational simplicity, we define
\begin{equation*}
    \begin{split}
        \alpha &= 2 \sqrt{\frac{\log(32/\delta)}{n}} \\
        \rho &= \sqrt{\frac{1}{n}} + 2 \sqrt{\frac{\log(16/\delta)}{n}}
    .\end{split}
\end{equation*}
By a union bound, the following collection of events occurs with probability at least $1 - \delta/2$:

\begin{enumerate}
    \item (Approximate Orthogonality.) By \cref{lem:lowrank-projection}, it holds that
    \begin{equation} \label{eqn:orthogonal} 
        | \langle \xi,G\rangle| < \norm{\xi}_2 \norm{G}_2 \rho
    .\end{equation}
    
    \item (Typical Norm of $G$ and $\xi$.) By \cref{lem:norm-concentration}, it holds that
    \begin{equation}\label{eqn:g-norm2}
         - \alpha \le \frac{1}{\sqrt{n}}\|G\|_2 - 1 \le \alpha
    \end{equation}
    and
    \begin{equation}\label{eqn:xi-norm2}
         - \alpha \sigma \le \frac{1}{\sqrt{n}}\|\xi\|_2 - \sigma \le \alpha \sigma.
    \end{equation}
    
    \item (Typical Norm of $\Sigma_2^{1/2}H$.) By \cref{thm:gaussian-concentration}, it holds that
    \begin{equation} \label{eqn:gp-concentration2} 
        \begin{split}
            \| \Sigma_2^{1/2}H \|_* &\geq \E \| \Sigma_2^{1/2}H \|_* - \sup_{\| u\| \leq 1} \| u \|_{\Sigma_2} \sqrt{2 \log(8/\delta)} \\
            &= \left(1 - \sqrt{\frac{2 \log(8/\delta)}{r_{\| \cdot \|}(\Sigma)}} \right) \E \| \Sigma_2^{1/2}H \|_*
        ,\end{split}
    \end{equation}
    because $\| \Sigma_2^{1/2}H \|_*$ is a $\sup_{\| u\| \leq 1} \| u \|_{\Sigma_2}$-Lipschitz function of $H$.
\end{enumerate}
By a change of variables, recall that 
\begin{equation*} 
    \phi := \min_{\norm{ \xi-  \Norm{\Sigma^{1/2} w}_2 G }_2 \leq \langle H, \Sigma^{1/2} w \rangle} \|  w \|.
\end{equation*}
\Cref{eqn:orthogonal,eqn:g-norm2,eqn:xi-norm2} imply that
\begin{equation*}
    \begin{split}
        \norm{ \xi-  \|\Sigma^{1/2} w\|_2 G }_2^2 &= \|\xi\|_2^2 - 2 \langle \xi, G \rangle \|\Sigma^{1/2} w\|_2 + \|\Sigma^{1/2} w\|_2^2 \|G\|_2^2 \\
        &\leq (1+\rho) \left(\|\xi\|_2^2+  \|\Sigma^{1/2} w\|_2^2 \|G\|_2^2 \right)\\
        &\leq (1+\rho)(1+\alpha)^2 n (\sigma^2 + \|\Sigma^{1/2} w\|_2^2).
    \end{split}
\end{equation*}
To upper bound $\phi$, it suffices to construct a $w$ that satisfies the constraint. Consider $w$ of the form $s(Pv^*)$, then $\Sigma^{1/2} w = s\Sigma^{1/2}_2 v^*$. Plugging in, it suffices to choose $s$ such that 
\begin{equation*}
    \begin{split}
        (1+\rho)(1+\alpha)^2 n (\sigma^2 + s^2\|\Sigma^{1/2}_2 v^*\|_2^2) &\leq s^2 \langle H, \Sigma^{1/2}_2 v^*\rangle^2 = s^2 \| \Sigma^{1/2}_2 H \|_*^2
    .\end{split}
\end{equation*}
Solving for $s$, we can choose 
\begin{equation*}
    s^2 = \sigma^2 \left( \frac{\| \Sigma^{1/2}_2 H \|_*^2}{(1+\rho)(1+\alpha)^2 n}- \| v^*\|_{\Sigma_2}^2 \right)^{-1}
\end{equation*}
given that it is positive. By \eqref{eqn:norm-auxiliary1} and \eqref{eqn:gp-concentration2}, we have
\begin{equation*}
    \begin{split}
        &\frac{\| \Sigma^{1/2}_2 H \|_*^2}{(1+\rho)(1+\alpha)^2 n}- \| v^*\|_{\Sigma_2}^2 \\
        \geq \, &\frac{(\E \| \Sigma^{1/2}_2 H \|_*)^2}{(1+\rho)(1+\alpha)^2 n} \left(1 - \sqrt{\frac{2 \log(8/\delta)}{r_{\| \cdot \|}(\Sigma)}} \right)^2 -(1+\epsilon_1)^2 (\E \norm{v^*}_{\Sigma_2} )^2\\
        = \, &\frac{(\E \| \Sigma^{1/2}_2 H \|_*)^2}{n} \left( \frac{1}{(1+\rho)(1+\alpha)^2} \left(1 - 2\sqrt{\frac{2 \log(8/\delta)}{r_{\| \cdot \|}(\Sigma)}} \right) - (1+\epsilon_1)^2 \frac{n}{R_{\| \cdot \|}(\Sigma_2)}\right)
    .\end{split}
\end{equation*}
If $\alpha < 1$, then
\begin{equation*}
    \begin{split}
        (1+\rho)(1+\alpha)^2 &= (1+\rho)(1+2\alpha+\alpha^2)\\
        &\leq (1+\rho)(1+3\alpha) = 1 + 3\alpha + \rho  + 3\alpha \rho \\
        &\leq 1 + 3\alpha +4\rho \\
    \end{split}
\end{equation*}
and using the inequality $(1+x)^{-1} \geq 1-x$, we show
\begin{equation*}
    \begin{split}
        \frac{1}{(1+\rho)(1+\alpha)^2} &\geq 1 - ((1+\rho)(1+\alpha)^2-1)\\
        &\geq 1-(3\alpha +4\rho).\\
    \end{split}
\end{equation*}
Therefore, we can conclude that 
\begin{equation*}
    \begin{split}
        & \frac{1}{(1+\rho)(1+\alpha)^2} \left(1 - 2\sqrt{\frac{2 \log(8/\delta)}{r_{\| \cdot \|}(\Sigma)}} \right) - (1+\epsilon_1)^2 \frac{n}{R_{\| \cdot \|}(\Sigma_2)} \\
        \geq \, & (1-(3\alpha +4\rho)) \left(1 - 2\sqrt{\frac{2 \log(8/\delta)}{r_{\| \cdot \|}(\Sigma)}} \right) - (1+\epsilon_1)^2 \frac{n}{R_{\| \cdot \|}(\Sigma_2)} \\
        \geq \, & 1 - (3\alpha +4\rho) - 2\sqrt{\frac{2 \log(8/\delta)}{r_{\| \cdot \|}(\Sigma)}} - (1+\epsilon_1)^2 \frac{n}{R_{\| \cdot \|}(\Sigma_2)} \geq 1- \epsilon^{\prime}\\
    \end{split}
\end{equation*}
where we define
\begin{equation*}
    \epsilon^{\prime} = 4 n^{-1/2} + 14 \sqrt{\frac{\log(32/\delta)}{n}} + 4\sqrt{\frac{\log(8/\delta)}{r_{\| \cdot \|}(\Sigma)}} + (1+\epsilon_1)^2 \frac{n}{R_{\| \cdot \|}(\Sigma_2)}.
\end{equation*}
Provided that $\epsilon^{\prime} \leq 1/2$ (which also guarantees that $\alpha < 1$ and our definition of $s^2$ is sensible), we can use the inequality $(1-x)^{-1} \leq 1 + 2x$ for $x \in [0, 1/2]$ to show that
\begin{equation*}
    s^2 \leq \sigma^2 \frac{n}{(\E \| \Sigma^{1/2}_2 H \|_*)^2} \frac{1}{1-\epsilon^{\prime}} \leq (1+2\epsilon^{\prime}) \, \sigma^2 \frac{n}{(\E \| \Sigma^{1/2}_2 H \|_*)^2}
\end{equation*}
and thus by \eqref{eqn:norm-auxiliary2}
\begin{equation*}
    \phi^2 \leq s^2 \| P v^* \|^2 \leq (1+\epsilon_2)(1+2\epsilon^{\prime}) \, \sigma^2 \frac{n}{(\E \| \Sigma^{1/2}_2 H \|_*)^2} \leq (1+\epsilon) \, \sigma^2 \frac{n}{(\E \| \Sigma^{1/2}_2 H \|_*)^2}
\end{equation*}
with $\epsilon = 2\epsilon^{\prime} + 2 \epsilon_2$.
\end{proof}

Finally, we are ready to prove our general norm bound.

\gennormbound*

\begin{proof}
By \cref{lem:rewrite-norm,lem:gmt-app-norm}, we show that for any $t$
\begin{equation*}
    \Pr (\norm{\hat w} > t) \leq \Pr (\Phi > t - \norm{w^*} ) \leq 2\Pr (\phi \geq t - \norm{w^*} )
.\end{equation*}
By \cref{lem:norm-auxiliary}, the above is upper bounded by $\delta$ if we set $t - \norm{w^*} $ according to \eqref{eqn:aux-bound2} with $\delta$ replaced by $\delta/2$. Moving $\norm{w^*}$ to the other side concludes the proof.
\end{proof}

\subsection{Special Case: Euclidean Norm} \label{apdx:norm-bounds:euclidean}

\begin{lemma} \label{lem:rank-fact}
For any covariance matrix $\Sigma$, it holds that 
\begin{equation}\label{eqn:norm-is-big}
\left( \E \|\Sigma^{1/2} H\|_2 \right)^2 \geq \left(1 - \frac{1}{r(\Sigma)}\right) \Tr(\Sigma) 
\end{equation}
and 
\begin{equation} \label{eqn:invnorm-is-small}
\frac{1}{\Tr(\Sigma)} \geq \left( 1 - \sqrt{\frac{8}{r(\Sigma)}}\right) \E \left[ \frac{1}{H^T \Sigma H} \right] .
\end{equation}

As a result, it holds that
\begin{equation} \label{eq:r-defs-close}
    r(\Sigma) - 1 
    \leq r_{\| \cdot \|_2} (\Sigma) 
    \leq r(\Sigma)
\end{equation}
and
\begin{equation} \label{eq:big-r-defs-close}
     1 - \frac{4}{\sqrt{r(\Sigma)}}
    \leq \frac{R_{\| \cdot \|_2}(\Sigma) }{R(\Sigma)} 
    \leq \left( 1 - \sqrt{\frac{8}{r(\Sigma^2)}}\right)^{-1} .
\end{equation}
\end{lemma}

\begin{proof}
Observe that if $f(H) = \| \Sigma^{1/2} H \|_2$, then it can easily be checked that
\begin{equation*}
    \| \nabla f \|_2^2 = \frac{\| \Sigma H \|^2_2}{\| \Sigma^{1/2} H \|_2^2} \leq \| \Sigma \|_{\OP}
\end{equation*}
and so by the Gaussian Poincar\'e inequality \citep[Corollary 2.27]{van2014probability}, we have
\begin{equation*}
    \begin{split}
        \Tr(\Sigma) = \E \|\Sigma^{1/2} H\|_2^2
        &= (\E \|\Sigma^{1/2} H\|_2)^2 + \Var \|\Sigma^{1/2} H\|_2 \\
        &\leq (\E \|\Sigma^{1/2} H\|_2)^2 + \|\Sigma\|_{\OP} \\
        &= (\E \|\Sigma^{1/2} H\|_2)^2 + \frac{\Tr(\Sigma)}{r(\Sigma)}
    .\end{split}
\end{equation*}
Rearranging the terms proves \eqref{eqn:norm-is-big}. To prove \eqref{eqn:invnorm-is-small}, without loss of generality assume that $\Sigma$ is diagonal, with diagonal entries $\lambda_1 \geq \lambda_2 \geq ... \geq \lambda_d$. Observe that for any integer $\nu > 2$, we can pad $\Sigma$ with 0's such that $\nu$ divides $d$, and we have
\begin{equation*}
    H^T \Sigma H = \sum_{i=1}^d \lambda_i H_i^2 \geq \sum_{i=1}^{d/\nu} \lambda_{\nu i} (H_{\nu(i-1)+1}^2 + ... + H_{\nu i}^2)
.\end{equation*}
By Jensen's inequality, $1 / \E [X] \leq \E [1/X] $; it follows that
\begin{equation*}
    \begin{split}
        \E \left[ \frac{1}{H^T \Sigma H} \right] &\leq \E \left[ \frac{1}{\sum_{i=1}^{d/\nu} \lambda_{\nu i} (H_{\nu(i-1)+1}^2 + ... + H_{\nu i}^2)} \right]\\
        &= \frac{1}{\sum_{j=1}^{d/\nu} \lambda_{\nu j}} \E \left[ \frac{1}{\sum_{i=1}^{d/\nu} \frac{\lambda_{\nu i}}{\sum_{j=1}^{d/\nu} \lambda_{\nu j}}  (H_{\nu(i-1)+1}^2 + ... + H_{\nu i}^2)} \right] \\
        &\leq \frac{1}{\sum_{j=1}^{d/\nu} \lambda_{\nu j}} \E \left[ \sum_{i=1}^{d/\nu} \frac{\lambda_{\nu i}}{\sum_{j=1}^{d/\nu} \lambda_{\nu j}} \frac{1}{ H_{\nu(i-1)+1}^2 + ... + H_{\nu i}^2} \right] \\
        &= \frac{1}{\sum_{j=1}^{d/\nu} \lambda_{\nu j}} \frac{1}{\nu-2}
    .\end{split}
\end{equation*}
In the last equality, we use the fact that for each $i$ the random variable $\left( H_{\nu(i-1)+1}^2 + ... + H_{\nu i}^2 \right)^{-1}$ follows an inverse Chi-square distribution with $\nu$ degrees of freedom; its expectation is $(\nu-2)^{-1}$. In addition, notice that
\begin{equation*}
    \nu \| \Sigma\|_{\OP} + \nu \sum_{i=1}^{d/\nu} \lambda_{\nu i} \geq (\lambda_1 + ... + \lambda_{\nu}) + \sum_{i=1}^{d/\nu-1} (\lambda_{\nu i + 1} + ... + \lambda_{\nu (i + 1)}) = \Tr(\Sigma)
.\end{equation*}
Plugging the above estimate into our upper bound shows for any integer $\nu > 2$, it holds that
\begin{equation*}
    \E \left[ \frac{1}{H^T \Sigma H} \right] \leq \frac{1}{\Tr(\Sigma)-\nu \| \Sigma\|_{\OP}} \frac{\nu}{\nu-2} = \frac{1}{\Tr(\Sigma)} \left( 1 - \frac{\nu}{r(\Sigma)} - \frac{2}{\nu} +\frac{2}{r(\Sigma)}\right)^{-1}
.\end{equation*}
We can show \eqref{eqn:invnorm-is-small} by choosing $\nu = \lceil (2 r(\Sigma))^{1/2} \rceil$:
\begin{equation*}
    \E \left[ \frac{1}{H^T \Sigma H} \right] \leq \frac{1}{\Tr(\Sigma)} \left( 1 - \sqrt{\frac{8}{r(\Sigma)}}\right)^{-1}
.\end{equation*}

It remains to verify \eqref{eq:r-defs-close} and \eqref{eq:big-r-defs-close}. By \eqref{eqn:norm-is-big}, we can check
\begin{equation*} 
    r_{\| \cdot \|_2} (\Sigma) = \frac{(\E \| \Sigma^{1/2} H \|_2)^2}{\| \Sigma \|_{\OP}} \geq \left(1 - \frac{1}{r(\Sigma)}\right) \frac{\Tr(\Sigma)}{\| \Sigma \|_{\OP}} = r(\Sigma) - 1.
\end{equation*}
The other direction $r(\Sigma) \geq r_{\| \cdot \|_2}(\Sigma)$ follows directly from an application of the Cauchy-Schwarz inequality. By Jensen's inequality $1 / \E [X] \leq \E [1/X] $ and the Cauchy-Schwarz inequality, we show
\begin{equation*}
    \frac{1}{\Tr(\Sigma)} \left( \E \frac{1}{\| \Sigma H \|_2^2} \right)^{-1} 
    \leq \left( \E \frac{\| \Sigma^{1/2} H \|_2}{\| \Sigma H \|_2 } \right)^{-2} 
    \leq \left( \E \frac{\| \Sigma H \|_2}{\| \Sigma^{1/2} H \|_2} \right)^2 
    \leq \Tr(\Sigma^2) \E \frac{1}{\| \Sigma^{1/2} H \|_2^2}
.\end{equation*}
Recall that $ R_{\| \cdot \|_2}(\Sigma) = (\E \| \Sigma^{1/2} H \|_2)^2 \left( \E \frac{\| \Sigma H \|_2}{\| \Sigma^{1/2} H \|_2} \right)^{-2}$. By Cauchy-Schwarz inequality and \eqref{eqn:invnorm-is-small}, it follows that
\begin{equation*}
    R_{\| \cdot \|_2}(\Sigma) \leq \Tr(\Sigma)^2 \left( \E \frac{1}{\| \Sigma H \|_2^2} \right)
    \leq \left( 1 - \sqrt{\frac{8}{r(\Sigma^2)}}\right)^{-1} R(\Sigma)
\end{equation*}
and also by \eqref{eqn:norm-is-big}
\begin{equation*}
    \begin{split}
        R_{\| \cdot \|_2}(\Sigma) &\geq \left(1 - \frac{1}{r(\Sigma)}\right)\frac{\Tr(\Sigma)}{\Tr(\Sigma^2)} \left( \E \frac{1}{\| \Sigma^{1/2} H \|_2^2}  \right)^{-1}\\
        &\geq \left(1 - \frac{1}{r(\Sigma)}\right) \left( 1 - \sqrt{\frac{8}{r(\Sigma)}}\right) R(\Sigma) 
        \geq \left( 1 - \frac{4}{\sqrt{r(\Sigma)}}\right) R(\Sigma).
    \qedhere \end{split}
\end{equation*}
\end{proof}

\begin{lemma}\label{lem:sigmah-concentration}
For any covariance matrix $\Sigma$, it holds that with probability at least $1 - \delta$,
\begin{equation}
1- \frac{\| \Sigma^{1/2} H\|_2^2 }{\Tr(\Sigma)}  \lesssim \frac{\log (4/\delta)}{\sqrt{R(\Sigma)}}
\end{equation}
and
\begin{equation}
\| \Sigma H\|_2^2 \lesssim \log (4/\delta) \Tr(\Sigma^2).
\end{equation}

Therefore, provided that $ R(\Sigma) \gtrsim \log (4/\delta)^2$, it holds that
\begin{equation} \label{eqn:v*-norm}
    \left( \frac{\| \Sigma H\|_2}{\| \Sigma^{1/2} H\|_2} \right)^2 \lesssim \log (4/\delta) \frac{\Tr(\Sigma^2)}{\Tr(\Sigma)}
.\end{equation}

\end{lemma}

\begin{proof}
Because we are considering $\ell_2$ norm and $H$ is standard Gaussian, without loss of generality we can assume that $\Sigma$ is diagonal and we denote the diagonals of $\Sigma$ as $\lambda_1, ..., \lambda_d$. By the sub-exponential Bernstein inequality \citep[Corollary 2.8.3]{vershynin2018high}, we have with probability at least $1-\delta/2$
\begin{equation*} 
    \left| \frac{\| \Sigma^{1/2} H\|_2^2 }{\Tr(\Sigma)} - 1 \right| = \left| \sum_{i=1}^p \frac{\lambda_i}{\sum_j \lambda_j} (H_i^2 - 1) \right| \lesssim \sqrt{\frac{\log (4/\delta)}{R(\Sigma)}} \vee \frac{\log (4/\delta)}{r(\Sigma)} \leq  \frac{\log (4/\delta)}{\sqrt{R(\Sigma)}}
\end{equation*}
where the last inequality uses that $R(\Sigma) \leq r(\Sigma)^2$, shown in Lemma 5 of \citet{bartlett2020benign}. Using the sub-exponential Bernstein inequality again, we show with probability at least $1-\delta/2$
\begin{equation*}
    \left| \frac{\| \Sigma H\|_2^2 }{\Tr(\Sigma^2)} - 1 \right| \lesssim \sqrt{\frac{\log (4/\delta)}{R(\Sigma^2)}} \vee \frac{\log (4/\delta)}{r(\Sigma^2)}
\end{equation*}
From Lemma 5 of \citet{bartlett2020benign}, we know that the effective ranks are at least 1. This implies 
\begin{equation*}
    \| \Sigma H\|_2^2 \lesssim \log (4/\delta) \Tr(\Sigma^2).
\end{equation*}
Provided that $ R(\Sigma) \gtrsim \log (4/\delta)^2$, we have 
\begin{equation*}
    \| \Sigma^{1/2} H\|_2^2 \geq \frac{1}{2} \Tr(\Sigma)
\end{equation*}
in which case it holds that
\begin{equation*}
     \frac{\| \Sigma H\|_2^2}{\| \Sigma^{1/2} H\|_2^2}  \lesssim \log (4/\delta) \frac{\Tr(\Sigma^2)}{\Tr(\Sigma)}
.\qedhere\end{equation*}
\end{proof}

\euclidnormbound*

\begin{proof}
To apply \cref{thm:interpolator-general}, it is clear that $v^* = \frac{\Sigma_2^{1/2} H}{\| \Sigma_2^{1/2} H \|_2}$ and so $\| v^* \|_{\Sigma_2} =  \frac{\| \Sigma_2 H\|_2}{\| \Sigma_2^{1/2} H\|_2}$. By \eqref{eqn:v*-norm}, it suffices to pick $\epsilon_1$ such that for some constant $c > 0$
\begin{equation*}
    (1+\epsilon_1) \E \norm{v^*}_{\Sigma_2} = c \sqrt{\log (16/\delta) \frac{\Tr(\Sigma_2^2)}{\Tr(\Sigma_2)}}.
\end{equation*}

By \eqref{eqn:norm-is-big} of \cref{lem:rank-fact}, for sufficiently large effective rank, it holds that $\left( \E \|\Sigma_2^{1/2} H\|_2 \right)^2 \gtrsim \Tr(\Sigma_2)$ and so
\begin{equation*}
    (1+\epsilon_1)^2 \frac{n}{R_{\|\cdot \|_2}(\Sigma_2)} = n \frac{(1+\epsilon_1)^2 (\E \| v^* \|_{\Sigma_2})^2}{\left( \E \|\Sigma_2^{1/2} H\|_2 \right)^2} \lesssim n \log (16/\delta) \frac{\Tr(\Sigma_2^2)}{\Tr(\Sigma_2)^2} = \frac{n \log (16/\delta)}{R(\Sigma_2)}.
\end{equation*}

Furthermore, it suffices to let $\epsilon_2 = 0$ because $P$ is an $\ell_2$ projection matrix. Combined with \eqref{eq:r-defs-close} of \cref{lem:rank-fact}, we show
\begin{equation*}
    \epsilon \lesssim \sqrt{\frac{\log(1/\delta)}{n}} + \sqrt{\frac{\log(1/\delta)}{r(\Sigma_2)}} + \frac{n \log (1/\delta)}{R(\Sigma_2)}.
\end{equation*}
Finally, using the inequality $(1-x)^{-1} \leq 1 + 2x$ for $x \in [0, 1/2]$ and \eqref{eqn:norm-is-big} of \cref{lem:rank-fact} again, we can conclude
\begin{equation*}
    \begin{split}
        (1+\epsilon)^{1/2} \, \sigma \frac{\sqrt{n}}{\E \| \Sigma^{1/2}_2 H \|_2} &\leq (1+\epsilon)^{1/2} \left(1 - \frac{1}{r(\Sigma_2)}\right)^{-1/2} \sigma \sqrt{\frac{n}{\Tr(\Sigma_2)}}\\
        &\leq \left(1+2\epsilon + \frac{2}{r(\Sigma_2)}\right)^{1/2} \sigma \sqrt{\frac{n}{\Tr(\Sigma_2)}}\\
    \end{split}
\end{equation*}
and we can replace $\epsilon$ with 
\begin{equation*}
    \epsilon^{\prime} = 2\epsilon + \frac{2}{r(\Sigma_2)} \lesssim \sqrt{\frac{\log(1/\delta)}{n}} + \sqrt{\frac{\log(1/\delta)}{r(\Sigma_2)}} + \frac{n \log (1/\delta)}{R(\Sigma_2)}.
\end{equation*}
\end{proof}

\section{Benign Overfitting} \label{apdx:benign-overfitting}

In this section, we will combine results from the previous two sections to study when interpolators are consistent.

\subsection{General Norm}

\minnorm*

\begin{proof}
By \cref{thm:interpolator-general}, if we choose 
\begin{equation*}
    B = \norm{w^*} + (1+\epsilon)^{1/2} \, \sigma \frac{\sqrt{n}}{\E \| \Sigma^{1/2}_2 H \|_*} 
\end{equation*}
then with large probability, $\{ w: \|w\| \le B \}$ has non-empty intersection with $\{ w: Xw = Y \}$, which contains the minimal norm interpolator $\hat w$. Also, it is clear that $B > \| w^* \|$ and so by \cref{corr:main-generalization}, it holds that

\begin{equation*}
    \begin{split}
        L(\hat w) &\leq \sup_{\|w\| \le B, \hat L(w) = 0} L(w) \\
        &\leq (1+\gamma) \left( \norm{w^*} + (1+\epsilon)^{1/2} \, \sigma \frac{\sqrt{n}}{\E \| \Sigma^{1/2}_2 H \|_*} \right)^2 \frac{\left( \E \Norm{ \Sigma^{1/2}_2 H }_* \right)^2}{n} \\
        &= (1+\gamma) \left( \norm{w^*} \frac{ \E \Norm{ \Sigma^{1/2}_2 H }_* }{\sqrt{n}} + (1+\epsilon)^{1/2} \, \sigma \right)^2  \\
        &\leq (1+\gamma)(1+\epsilon) \left( \sigma + \| w^* \| \frac{\E \|\Sigma_2^{1/2} H\|_*}{\sqrt{n}} \right)^2. \\
    \end{split}
\end{equation*}

\end{proof}

\begin{theorem}[Sufficient conditions] \label{thm:sufficient-general}
Under the model assumptions in \eqref{eqn:model}, let $\norm\cdot$ be an arbitrary norm. Suppose that as $n$ goes to $\infty$, there exists a sequence of covariance splits $\Sigma = \Sigma_1 \oplus \Sigma_2$ such that the following properties hold:
\begin{enumerate}
    \item (Small large-variance dimension.) 
    \begin{equation}
        \lim_{n \to \infty} \frac{\rank(\Sigma_1)}{n} = 0.
    \end{equation}
    
    \item (Large effective dimension.)
    \begin{equation}
        \lim_{n \to \infty} \frac{1}{r_{\|\cdot\|}(\Sigma_2)} = 0 \quad \text{and} \quad \lim_{n \to \infty} \frac{n}{R_{\|\cdot\|}(\Sigma_2)} = 0.
    \end{equation}
    
    \item (No aliasing condition.)
    \begin{equation}
        \lim_{n \to \infty} \frac{\norm{w^*} \E \| \Sigma^{1/2}_2 H \|_*}{\sqrt{n}} = 0 .
    \end{equation}
    
    \item (Contracting $\ell_2$ projection condition.)  With the same definition of $P$ and $v^*$ as in \cref{thm:interpolator-general}, it holds that for any $\eta > 0$,
    \begin{equation}
        \lim_{n \to \infty} \Pr(\| P v^* \|^2 > 1 + \eta) = 0 .
    \end{equation}
\end{enumerate}

Then $L(\hat w)$ converges to $\sigma^2$ in probability. In other words, minimum norm interpolation is consistent.
\end{theorem}

\begin{proof}
Fix any $\eta > 0$, for sufficiently small $\gamma, \epsilon$ and $\| w^* \| \frac{\E \|\Sigma_2^{1/2} H\|_*}{\sqrt{n}}$, it is clear that
\begin{equation} \label{eqn:sufficient-diff}
    (1+\gamma)(1+\epsilon) \left( \sigma + \| w^* \| \frac{\E \|\Sigma_2^{1/2} H\|_*}{\sqrt{n}} \right)^2 - \sigma^2 \leq \eta
.\end{equation}
For any $\delta > 0$, by the definition of $\gamma$ in \cref{corr:main-generalization} and our assumptions, the terms $\gamma$ and $\| w^* \| \frac{\E \|\Sigma_2^{1/2} H\|_*}{\sqrt{n}}$ can be made arbitrarily small for large enough $n$. Also by our assumption, $\epsilon_2$ in the definition of $\epsilon$ in \cref{thm:interpolator-general} can be arbitrarily small. Note that 
\begin{equation*}
    \sqrt{\frac{n}{R_{\|\cdot\|}(\Sigma_2)}} = \E \left[ \frac{\| v^* \|_{\Sigma_2}}{\E \| \Sigma_2^{1/2} H\|_* / \sqrt{n}} \right] 
\end{equation*}
converges to 0 by assumption. Then by Markov's inequality, for any $\eta^{\prime} > 0$, it holds that for all sufficiently large $n$
\begin{equation*}
    \Pr \left( \frac{\| v^* \|_{\Sigma_2}}{\E \| \Sigma_2^{1/2} H\|_* / \sqrt{n}} > \sqrt{\eta^{\prime}} \right) < \delta
\end{equation*}
and we can pick
\begin{equation*}
    (1+\epsilon_1) \E \norm{v^*}_{\Sigma_2} = \sqrt{\eta^{\prime}} \frac{\E \| \Sigma_2^{1/2} H\|_*}{\sqrt{n}}.
\end{equation*}
This implies that
\begin{equation*}
    (1+\epsilon_1)^2 \frac{n}{R_{\| \cdot \|}(\Sigma_2)} = \frac{n}{\left(\E \| \Sigma_2^{1/2} H\|_* \right)^2} \left( (1+\epsilon_1) \E \| v^* \|_{\Sigma_2} \right)^2 = \eta^{\prime}.
\end{equation*}
By \cref{thm:min-norm}, we have shown that for sufficiently large $n$ such that $\gamma, \epsilon$ and $\| w^* \| \frac{\E \|\Sigma_2^{1/2} H\|_*}{\sqrt{n}}$ are small enough for \eqref{eqn:sufficient-diff} to hold, it holds that
\begin{equation*}
    \Pr ( \, |L(\hat w) - \sigma^2| > \eta \,) \leq \delta.
\end{equation*}
As a result, we show $\lim_{n \to \infty} \Pr ( \, |L(\hat w) - \sigma^2| > \eta \,) \leq \delta$ for any $\delta > 0$. To summarize, for any fixed $\eta > 0$, we have
\begin{equation*}
    \lim_{n \to \infty} \Pr ( \, |L(\hat w) - \sigma^2| > \eta \,) = 0
\end{equation*}
and so $L(\hat w)$ converges to $\sigma^2$ in probability.
\end{proof}

\subsection{Euclidean Norm}

\minnormeuclid*

\begin{proof}
The proof follows the same strategy as \cref{thm:min-norm}. By \cref{thm:interpolator-euclidean}, if we choose 
\begin{equation*}
    B = \|w^*\|_2 + (1+\epsilon)^{1/2} \, \sigma \sqrt{\frac{n}{\Tr (\Sigma_2)}}
,\end{equation*}
then with large probability, $\{ w: \|w\|_2 \leq B \}$ has non-empty intersection with $\{ w: Xw = Y \}$. This intersection necessarily contains the minimal norm interpolator $\hat w$. 

Also, it is clear that $B > \| w^* \|$ and so by \cref{corr:euclidean-generalization}, it holds that
\begin{equation*}
    \begin{split}
        L(\hat w) &\leq \sup_{\|w\|_2 \leq B, \hat L(w) = 0} L(w) \\
        &\leq (1+\gamma) \left( \|w^*\|_2 + (1+\epsilon)^{1/2} \, \sigma \sqrt{\frac{n}{\Tr (\Sigma_2)}} \right)^2 \frac{\Tr(\Sigma_2)}{n} \\
        &= (1+\gamma) \left( \norm{w^*}_2 \sqrt{\frac{\Tr(\Sigma_2)}{n}} + (1+\epsilon)^{1/2} \, \sigma \right)^2  \\
        &\leq (1+\gamma)(1+\epsilon) \left( \sigma + \| w^* \|_2 \sqrt{\frac{\Tr(\Sigma_2)}{n}} \right)^2.
    \qedhere \end{split}
\end{equation*}
\end{proof}

\begin{theorem}[Sufficient conditions] \label{thm:sufficient-euclidean}
Under the model assumptions in \eqref{eqn:model}, let $\hat w$ be the minimal $\ell_2$ norm interpolator. Suppose that as $n$ goes to $\infty$, there exists a sequence of covariance splitting $\Sigma = \Sigma_1 \oplus \Sigma_2$ such that the following conditions hold:
\begin{enumerate}
    \item (Small large-variance dimension.) 
    \begin{equation}
        \lim_{n \to \infty} \frac{\rank(\Sigma_1)}{n} = 0.
    \end{equation}
    
    \item (Large effective dimension.)
    \begin{equation}
        \lim_{n \to \infty} \frac{n}{R(\Sigma_2)} = 0.
    \end{equation}
    
    \item (No aliasing condition.) 
    \begin{equation}
        \lim_{n \to \infty} \frac{\norm{w^*}_2 \E \| \Sigma^{1/2}_2 H \|_2}{\sqrt{n}} = 0.
    \end{equation}
\end{enumerate}
Then $L(\hat w)$ converges to $\sigma^2$ in probability. In other words, minimum $\ell_2$ norm interpolation is consistent.
\end{theorem}

\begin{proof}
Fix any $\eta > 0$, for sufficiently small $\gamma, \epsilon$ and $\| w^* \|_2 \sqrt{\frac{\Tr (\Sigma_2)}{n}}$, it is clear that
\begin{equation} \label{eqn:sufficient-diff2}
    (1+\gamma)(1+\epsilon) \left( \sigma + \| w^* \|_2 \sqrt{\frac{\Tr (\Sigma_2)}{n}} \right)^2 - \sigma^2 \leq \eta
.\end{equation}
From Lemma 5 of \citet{bartlett2020benign}, it holds that $R(\Sigma_2) \leq r(\Sigma_2)^2$, and so the condition $R(\Sigma_2) = \omega(n)$ implies that $r(\Sigma_2) = \omega(\sqrt{n}) = \omega(1)$. For any $\delta > 0$, by the definition of $\gamma, \epsilon$ in \cref{corr:euclidean-generalization} and \cref{thm:interpolator-euclidean} and our assumptions, the terms $\gamma, \epsilon$ and $\| w^* \|_2 \sqrt{\frac{\Tr (\Sigma_2)}{n}}$ can be made small enough for \cref{eqn:sufficient-diff2} to hold with a sufficiently large $n$. By \cref{corr:min-norm-euclidean}, we show that
\begin{equation*}
    \lim_{n \to \infty} \Pr ( \, |L(\hat w) - \sigma^2| > \eta \,) \leq \delta
\end{equation*}
Since the choice of $\delta > 0$ is arbitrary, we have shown that $L(\hat w)$ converges to $\sigma^2$ in probability.
\end{proof}

\subsubsection{Equivalence of Consistency Conditions} \label{apdx:equivalent}

If we assume that $\|w^*\| = \Theta(1)$, our consistency condition (\cref{thm:sufficient-euclidean}) for minimum $\ell_2$ norm interpolation is the existence of a covariance splitting such that 
\begin{equation}\label{eqn:asymptotic2}
    \rank(\Sigma_1) = o(n)
    , \qquad
    \Tr \Sigma_2 = o(n)
    , \qquad
    \frac{(\Tr \Sigma_2)^2}{\Tr[(\Sigma_2)^2]} = \omega(n)
.\end{equation}
We compare the above conditions to the following conditions:
\begin{equation}\label{eqn:asymptotic1}
    \rank(\Sigma_1) = o(n)
    , \quad
    \Tr \Sigma_2 = o(n)
    , \quad
    \frac{\Tr \Sigma_2}{\|\Sigma_2\|_{\OP}} = \omega(n)
    , \quad
    \frac{(\Tr \Sigma_2)^2}{\Tr[(\Sigma_2)^2]} = \omega(n) 
.\end{equation}

Obviously, the conditions in \eqref{eqn:asymptotic1} imply \eqref{eqn:asymptotic2}, but we show in \cref{thm:asymptotic-equivalent} that the existence of a splitting that satisfies \eqref{eqn:asymptotic2} also implies the existence of a (potentially different) splitting that satisfies \eqref{eqn:asymptotic1}. This is one way to see that the particular choice of $k^*$ from \citet{bartlett2020benign} can be made without loss of generality, at least if we only consider the consistency conditions.

\begin{theorem}\label{thm:asymptotic-equivalent}
Suppose that there exists $\Sigma = \Sigma_1 \oplus \Sigma_2$ that satisfies the conditions in \eqref{eqn:asymptotic2}. Then there exists a $\Sigma = \Sigma^{'}_1 \oplus \Sigma^{'}_2$ that satisfies the conditions in \eqref{eqn:asymptotic1}.
\end{theorem}

\begin{proof}
Denote $v$ as the vector of eigenvalues of $\Sigma$, and $v_k$ as the vector obtained by setting the $k$ coordinates of $v$ corresponding to $\Sigma_1$ to be 0. By our assumptions in \eqref{eqn:asymptotic2}, there exists $k = o(n)$ such that
\begin{equation*}
    \|v_k \|_1 = o(n)
    , \qquad
    \frac{\|v_{k}\|_1^2}{\|v_{k}\|_2^2} = \omega(n)
.\end{equation*}
For any $\tau \geq 0$, we let $S_{\tau} = \{ i \in [d] : |v_{k,i}| \geq \tau \|v_k\|_{\infty}\}$ and define $v_{k, \tau}$ by setting the coordinates of $v_k$ in $S_{\tau}$ to be 0. For simplicity of notation, define $a = \|v_{k}\|_1^2/\|v_{k}\|_2^2$ and $b = \|v_{k}\|_1/\|v_{k}\|_{\infty}$. Observe that
\begin{equation*}
    \sum_{i \in S_{\tau}} |v_{k,i}| \leq \frac{1}{\tau \|v_k\|_{\infty}} \sum_{i \in S_{\tau}} v_{k,i}^2 \leq \frac{\| v_k \|_2^2}{\tau \|v_k\|_{\infty}} = \frac{\|v_k\|_1}{\tau} \frac{b}{a}
.\end{equation*}
This shows that
\begin{equation*}
    \| v_{k, \tau} \|_1 \geq \left( 1 - \frac{b}{\tau a}\right) \|v_k\|_1
\end{equation*}
and 
\begin{equation*}
    \| v_{k, \tau} \|_{\infty} \leq \tau \| v_{k} \|_{\infty} = \frac{\tau}{b} \|v_k\|_1.
\end{equation*}
In addition, observe that
\begin{equation*}
    \tau \|v_k\|_{\infty} \cdot |S_{\tau}| \leq \sum_{i \in S_{\tau}} |v_{k,i}| \leq \frac{\|v_k\|_1}{\tau} \frac{b}{a}.
\end{equation*}
The above inequalities imply that
\begin{equation*}
    \frac{\| v_{k, \tau} \|_1}{\| v_{k, \tau} \|_{\infty}} \geq  \frac{b}{\tau} \left( 1 - \frac{b}{\tau a}\right)
\end{equation*}
and
\begin{equation*}
    |S_{\tau}| \leq a \cdot \left( \frac{b}{\tau a}\right)^2
\end{equation*}

Finally, we pick $\tau$ by setting $b/(\tau a) = (n/a)^{3/4}$. By our assumption that $a = \omega(n)$, we can check 
\begin{equation*}
    \frac{b}{\tau} \left( 1 - \frac{b}{\tau a}\right) = n^{3/4} a^{1/4} (1-(n/a)^{3/4}) = \omega(n) (1-o(1)) = \omega(n)
\end{equation*}
and
\begin{equation*}
    a \cdot \left( \frac{b}{\tau a}\right)^2 = a(n/a)^{3/2} = n (n/a)^{1/2} = o(n).
\end{equation*}
By Holder's inequality, we also have
\begin{equation*}
    \frac{\| v_{k, \tau} \|_1^2}{\| v_{k, \tau} \|_2^2} \geq \frac{\| v_{k, \tau} \|_1}{\| v_{k, \tau} \|_{\infty}} = \omega(n)
.\end{equation*}
It is clear that $\| v_{k, \tau} \|_1 \leq \| v_{k} \|_1 = o(n)$ and $k+|S_{\tau}| = o(n)$, so picking the covariance splitting that corresponds to $v_{k, \tau}$ concludes the proof.
\end{proof}

\section{Basis Pursuit (Minimum \texorpdfstring{$\ell_1$}{l1}-Norm Interpolation)} \label{apdx:bp}

In this section, we illustrate the consequences of our general theory for basis pursuit. The following generalization bound for basis pursuit follows immediately from \cref{corr:main-generalization}:

\begin{corollary}[Generalization bound for $\ell_1$ norm balls] \label{corr:BP-generalization}
There exists an absolute constant $C_1 \leq 66$ such that the following is true. Under the model assumptions in \eqref{eqn:model} with $\Sigma = \Sigma_1 \oplus \Sigma_2$, fix $\delta \le 1/4$ and let $\gamma = C_1 \left(\sqrt{\frac{\log(1/\delta)}{r_1(\Sigma_2)}} + \sqrt{\frac{\log(1/\delta)}{n}} + \sqrt{\frac{\rank(\Sigma_1) }{n}} \right)$. If $B \geq \| w^* \|_1$ and $n$ is large enough that $\gamma \leq 1$, then the following holds with probability at least $1 - \delta$:
\begin{equation}
    \sup_{\|w\|_1 \le B, \hat L(w) = 0} L(w)
 \leq (1+\gamma) \frac{\left( B \cdot \E \Norm{ \Sigma^{1/2}_2 H }_{\infty} \right)^2}{n}.
\end{equation}
\end{corollary}

\begin{proof}
Recall that the dual of the $\ell_1$ norm is the $\ell_{\infty}$ norm. By convexity
\begin{equation*} 
    \max_{\|w\|_1 \le 1} \|w\|_{\Sigma} =  \sqrt{\max_{i} \, \langle e_i, \Sigma e_i \rangle} = \sqrt{\max_i \, \Sigma_{ii}}
\end{equation*}
and so we can use $ r_1(\Sigma) = \frac{\left(\E \| \Sigma^{1/2} H \|_{\infty}\right)^2}{\max_i (\Sigma)_{ii}} = r_{\| \cdot \|_1}(\Sigma)$.
\end{proof}

The following norm bound for basis pursuit follows from \cref{thm:interpolator-general}:

\begin{corollary}[$\ell_1$ norm bound] \label{corr:interpolator-BP}
There exists an absolute constant $C_2 \leq 64$ such that the following is true.
Under the model assumptions in \eqref{eqn:model}, let $\Sigma = \Sigma_1 \oplus \Sigma_2$ such that $\Sigma_2$ is diagonal.
Fix $\delta \le 1/4$ and let $\epsilon = C_2 \left( \sqrt{\frac{\log(1/\delta)}{r_1(\Sigma_2)}}  + \sqrt{\frac{\log(1/\delta)}{n}} + \frac{n}{r_1(\Sigma_2)} \right)$.
Then if $n$ and the effective rank $r_1(\Sigma_2)$ are large enough that $\epsilon \leq 1$, with probability at least $1-\delta$, it holds that
\begin{equation}
    \norm{\hat w}_1 \leq \norm{w^*}_1 + (1+\epsilon)^{1/2} \, \sigma \frac{\sqrt{n}}{\E \| \Sigma^{1/2}_2 H \|_{\infty}}.
\end{equation}
\end{corollary}

\begin{proof}
Recall that $\partial \|u\|_* = \text{conv} \{ \operatorname{sign}(u_i) \, e_i : i \in \arg\max |u_i| \}$, where $\text{conv}(S)$ denotes the convex hull of $S$. By definition, it holds almost surely that
\begin{equation} \label{eqn:v*-norm-bp}
    \|v^* \|_{\Sigma_2} \le \max_{i \in [d]} \, \|e_i\|_{\Sigma} = \sqrt{\max_i \, \Sigma_{ii}},
\end{equation} 
and so we can pick $\epsilon_1$ such that
\begin{equation*}
    (1+\epsilon_1) \E \norm{v^*}_{\Sigma_2} = \sqrt{\max_i \, \Sigma_{ii}}
\end{equation*}
and
\begin{equation*}
    (1+\epsilon_1)^2 \frac{n}{R_{\|\cdot \|_1}(\Sigma_2)} = n \frac{(1+\epsilon_1)^2 (\E \| v^* \|_{\Sigma_2})^2}{\left( \E \|\Sigma_2^{1/2} H\|_{\infty} \right)^2} = \frac{n}{r_1(\Sigma_2)}.
\end{equation*}

In addition, since $\Sigma_2$ is diagonal, the coordinates of $\Sigma_2^{1/2} H$ that correspond to the zero diagonals of $\Sigma_2$ are 0. Therefore, $ v^*$ must also have zero entry in those coordinates. In other words, $ v^*$ lies in the span of $\Sigma_2$. As $P$ is the orthogonal projection onto the space spanned by $\Sigma_2$, this implies $Pv^* = v^*$, and so $\| Pv^* \|_1 = \| v^* \|_1 = 1$, so that we can take $\epsilon_2 = 0$. Plugging $\epsilon_1, \epsilon_2$ into \cref{thm:interpolator-general} concludes the proof.
\end{proof}

\begin{theorem}[Benign overfitting] \label{thm:min-norm-BP}

Fix any $\delta \leq 1/2$. Under the model assumptions in \eqref{eqn:model}, let $\Sigma = \Sigma_1 \oplus \Sigma_2$ such that $\Sigma_2$ is diagonal. Suppose that $n$ and the effective rank $r_1(\Sigma_2)$ are sufficiently large such that $\gamma, \epsilon \leq 1$ with the same choice of $\gamma$ and $\epsilon$ as in \cref{corr:BP-generalization,corr:interpolator-BP}.
Then, with probability at least $1-\delta$:
\begin{equation}
    L(\hat{w})
    \leq (1+\gamma)(1+\epsilon) \left( \sigma + \| w^* \|_1 \frac{\E \| \Sigma^{1/2}_2 H \|_{\infty} }{ \sqrt{n} } \right)^2
.\end{equation}
\end{theorem}

The proof of \cref{thm:min-norm-BP} uses \cref{corr:BP-generalization,corr:interpolator-BP}, and follows the same lines as in \cref{thm:min-norm}. The details are repetitive, so we omit writing them out in full here. As before, we can use the finite sample bound to deduce sufficient conditions for consistency.

\begin{theorem}[Sufficient conditions] \label{thm:sufficient-BP}
Under the model assumptions in \eqref{eqn:model}, let $\hat w$ be the minimal $\ell_1$ norm interpolator. Suppose that as $n$ goes to $\infty$, there exists a sequence of covariance splits $\Sigma = \Sigma_1 \oplus \Sigma_2$ such that $\Sigma_2$ is diagonal and the following conditions hold:
\begin{enumerate}
    \item (Small large-variance dimension.) 
    \begin{equation}
        \lim_{n \to \infty} \frac{\rank(\Sigma_1)}{n} = 0.
    \end{equation}
    
    \item (Large effective dimension.)
    \begin{equation}
        \lim_{n \to \infty} \frac{n}{r_1(\Sigma_2)} = 0.
    \end{equation}
    
    \item (No aliasing condition.)
    \begin{equation}
        \lim_{n \to \infty} \frac{\norm{w^*}_1 \E \| \Sigma^{1/2}_2 H \|_{\infty}}{\sqrt{n}} = 0.
    \end{equation}
\end{enumerate}
Then $L(\hat w)$ converges to $\sigma^2$ in probability. In other words, minimum $\ell_1$ norm interpolation is consistent.
\end{theorem}

Again, the proof of \cref{thm:sufficient-BP} is exactly analogous to \cref{thm:sufficient-euclidean}, so we omit the full proof here.

\subsection{Isotropic features}

\begin{theorem} \label{corr:interpolator-BP-iso}
There exists an absolute constant $C_3 \leq 140$ such that the following is true. Under the model assumptions in \eqref{eqn:model} with $\Sigma = I_d$, denote $S$ as the support of $w^*$. Fix $\delta \le 1/4$ and let $\epsilon = C_3 \left( \sqrt{\frac{\log(1/\delta)}{n}} + \sqrt{\frac{\log(1/\delta)}{\log (d-|S|)}} + \frac{n}{\log (d-|S|)} \right)$. Then if $n$ and $d$ are large enough that $\epsilon \leq 1$, the following holds with probability $1-\delta$
where $H^{\prime} \sim N(0, I_{d-|S|})$:
\begin{equation}
    \norm{\hat w}_1 \leq (1+\epsilon)^{1/2} \, (\sigma^2 + \|w^*\|_2^2)^{1/2} \, \frac{\sqrt{n}}{\E \| H^{\prime} \|_{\infty}}.
\end{equation}
\end{theorem}

\begin{proof}
Write $X = [X_S, X_{S^\mathsf{c}}]$, where $X_S$ is formed by selecting the columns of $X$ in $S$. Also let $\xi^{\prime} = X_S w_S^* + \xi$; then the entries of $\xi^{\prime}$ are i.i.d. $N(0, \sigma^2 + \| w^*\|_2^2 )$ and independent of $X_{S^\mathsf{c}}$. Observe that $Y = X_{S^\mathsf{c}} 0 + \xi^{\prime} $. By choosing $\Sigma_1 = 0$ in \cref{corr:interpolator-BP}, we show with large probability
\begin{equation*}
    \min_{X_{S^\mathsf{c}} w = Y} \| w \|_1 \leq  (1+\epsilon)^{1/2} \, (\sigma^2 + \|w^*\|_2^2)^{1/2} \, \frac{\sqrt{n}}{\E \| H^{\prime} \|_{\infty}}
\end{equation*}
for some $\epsilon \leq 64 \left( \sqrt{\frac{\log(1/\delta)}{n}} + \sqrt{\frac{\log(1/\delta)}{r_1(I_{d-|S|})}} + \frac{n}{r_1(I_{d-|S|})} \right)$. By the bound of \citet{bound-max}, it holds that
\begin{equation*}
    r_1(I_{d-|S|}) = \left( \E \| H^{\prime}\|_{\infty} \right)^2  \geq \frac{\log (d-|S|)}{\pi \log 2}
\end{equation*}
and so we can choose $C_3 \leq 64 \pi \log2 < 140$. Observe that if $X_{S^\mathsf{c}} w = Y$, then $X (0, w)^T = Y $ and $\| (0, w) \|_1 = \| w \|_1$. It follows that
\begin{equation*}
    \norm{\hat w}_1 = \min_{X w = Y} \| w \|_1 \leq \min_{X_{S^\mathsf{c}} w = Y} \| w \|_1
.\qedhere \end{equation*}
\end{proof}

\begin{theorem}
Under the model assumptions in \eqref{eqn:model} with $\Sigma = I_d$, fix any $\delta \leq 1/2$ and let $\eta = 368 \left(\sqrt{ \frac{\log (1/\delta)}{n} } + \sqrt{\frac{\log(1/\delta) +  \log |S|}{\log (d-|S|)}} + \frac{n}{\log (d-|S|)}\right)$. Suppose that $n$ and $d$ are large enough that $\eta \leq 1$. Then, with probability at least $1-\delta$,
\begin{equation}
    L(\hat{w})
    \leq (1 + \eta) (\sigma^2 + \|w^*\|_2^2)
.\end{equation}
\end{theorem}

\begin{proof}
By \cref{corr:interpolator-BP-iso}, if we choose 
\begin{equation*}
    B = (1+\epsilon)^{1/2} \, (\sigma^2 + \|w^*\|_2^2)^{1/2} \, \frac{\sqrt{n}}{\E \| H^{\prime} \|_{\infty}}
\end{equation*}
then with large probability, $\cK = \{ w: \|w\|_1 \leq B \}$ has non-empty intersection with $\{ w: Xw = Y \}$, which contains the minimal $\ell_1$ norm interpolator $\hat w$. It can be easily seen that
\begin{equation*}
    W(\cK) = B \E \| H \|_{\infty} \quad \text{and} \quad R(\cK) = B
\end{equation*}
and so by \cref{thm:main-generalization}, with large probability
\begin{equation*}
    \begin{split}
        L(\hat w) &\leq \sup_{\|w\|_2 \leq B, \hat L(w) = 0} L(w) \\
        &\leq \frac{1 + \beta}{n} \left( B \E \| H \|_{\infty} + B \sqrt{2\log \left(\frac{64}{\delta}\right)} + \|w^*\|_{2} \sqrt{2\log\left(\frac{64}{\delta}\right)} \right)^2\\
        &= \frac{1 + \beta}{n} B^2 (\E \| H \|_{\infty} )^2 \left( 1 + \gamma \right)^2\\
        &= (1 + \beta)(1+\epsilon) ( 1 + \gamma )^2 \left( \frac{\E \| H \|_{\infty}}{\E \| H^{\prime} \|_{\infty}}\right)^2  (\sigma^2 + \|w^*\|_2^2) \\
    \end{split}
\end{equation*}
where $\beta = 66 \sqrt{\log(1/\delta)/n}$ and $\gamma = \frac{\sqrt{2\log \left(\frac{64}{\delta}\right)}}{\E \| H \|_{\infty}} + \frac{\|w^*\|_{2} \sqrt{2\log\left(\frac{64}{\delta}\right)}}{B \E \| H \|_{\infty}} $. Observe that 
\begin{equation*}
    B \geq  \|w^*\|_2 \, \frac{\sqrt{n}}{\E \| H^{\prime} \|_{\infty}} \quad \text{and} \quad \E \| H \|_{\infty} \geq \E \| H^{\prime} \|_{\infty}.
\end{equation*}
Combined with the lower bound of \citet{bound-max}, we show
\begin{equation*}
    \gamma \leq \sqrt{ \frac{2 \pi \log (128/\delta)}{\log d} } + \sqrt{ \frac{2  \log (64/\delta)}{n} } \leq 8 \left( \sqrt{ \frac{\log (1/\delta)}{\log d} } + \sqrt{ \frac{\log (1/\delta)}{n} } \right)
.\end{equation*}
In addition, we have
\begin{equation*}
    \frac{\E \| H \|_{\infty}}{\E \| H^{\prime} \|_{\infty}} = 1 + \frac{\E \| H \|_{\infty} - \E \| H^{\prime} \|_{\infty}}{\E \| H^{\prime} \|_{\infty}} \leq 1 +\sqrt{\frac{2 \pi \log(2) \log |S|}{\log (d-|S|)}}.
\end{equation*}
Finally, it is a routine calculation to show
\begin{equation*}
    \begin{split}
        &(1 + \beta)(1+\epsilon) ( 1 + \gamma )^2 \left( \frac{\E \| H \|_{\infty}}{\E \| H^{\prime} \|_{\infty}}\right)^2 \\
        \leq \, &1 + 368 \left(\sqrt{ \frac{\log (1/\delta)}{n} } + \sqrt{\frac{\log(1/\delta) +  \log |S|}{\log (d-|S|)}} + \frac{n}{\log (d-|S|)}\right) = 1 + \eta\\
    \end{split}
\end{equation*}
using the inequality $(1+x)(1+y) \leq 1 + x + 2y$ for $x \leq 1$. 
\end{proof}

\end{document}